\documentclass{article}

\usepackage[preprint]{neurips_2026}

\usepackage[utf8]{inputenc}
\usepackage[T1]{fontenc}
\usepackage{hyperref}
\usepackage{url}
\usepackage{booktabs}
\usepackage{subcaption}
\usepackage{amsfonts}
\usepackage{nicefrac}
\usepackage{microtype}
\usepackage{xcolor}

\usepackage{amsmath, amssymb, amsthm, mathtools}
\usepackage{graphicx}
\graphicspath{{figures/}{./}}
\usepackage{enumitem}
\usepackage{bm}
\usepackage{tikz}
\usetikzlibrary{arrows.meta,positioning,calc}

\newtheorem{proposition}{Proposition}

\newtheorem{corollary}{Corollary}

\newtheorem{remark}{Remark}

\newcommand{\R}{\mathbb{R}}
\newcommand{\E}{\mathbb{E}}

\newcommand{\z}{\bm{z}}
\newcommand{\vv}{\bm{v}}
\newcommand{\zp}{\z^{\mathrm{prog}}}
\newcommand{\zc}{\z^{\mathrm{cont}}}

\title{Subspace-Decomposed JEPAs: Disentangling Progression and Content in Latent World Models}

\author{%
  Lucas Thil \\
  LIX, École Polytechnique \\
  IRT SystemX \\
  Palaiseau, France \\
  \texttt{thil@lix.polytechnique.fr} \\
  \And
  Jesse Read \\
  LIX, École Polytechnique \\
  Palaiseau, France \\
  \And
  Rim Kaddah \\
  IRT SystemX \\
  Palaiseau, France \\
  \And
  Guillaume Doquet \\
  Safran Tech \\
  Chateaufort, France \\
}

\begin{document}

\maketitle

\begin{abstract}
World models used for planning learn to predict how a scene evolves, yet they rarely hold a global view of how far a task has progressed. We argue that such a view should be part of the representation, and that it can be compact: describing a scene takes many dimensions, but where the scene stands on its way to the goal can be summarised by only a few. We propose SD-JEPA, a joint-embedding predictive world model whose latent is split into a small progression subspace and a large content subspace. A temporal triplet loss keeps consecutive frames of a trajectory close in the progression coordinates and separates different trajectories, while a distributional regulariser keeps the content coordinates from collapsing.
On four control benchmarks, the model plans as well as the same architecture without the split. In return, its latent becomes easier to read. A few percent of its dimensions are enough to predict progress within an episode, and the angle they form follows the task, moves backward when the agent backtracks and changes fastest around events such as grasps and releases. 
Surprise also becomes anchored in the trajectory: when the scene changes mid-episode, the angle spikes and then relocalizes on a new part of the progression, anchoring surprise in the agent’s own map of the task.
Code is available at \url{https://github.com/LucasStill/SD-JEPA}.
\end{abstract}

\begin{figure}[t]
\centering
\begin{tikzpicture}[
  font=\small, >={Stealth[length=2mm]},
  box/.style={draw, rounded corners=2pt, minimum height=7mm, inner sep=3pt, align=center},
  note/.style={font=\footnotesize, align=center, text width=3.3cm}]
  \node[box, fill=gray!8] (obs) {frame\\$o_t$};
  \node[box, fill=gray!8, right=6mm of obs] (enc) {encoder\\$f_\theta$};
  \node[draw, fill=green!45!black!25, minimum width=6mm, minimum height=7mm, right=6mm of enc, inner sep=0] (zp) {};
  \node[draw, fill=gray!20, minimum width=30mm, minimum height=7mm, right=0mm of zp, inner sep=0] (zc) {$\zc_t \in \R^{D-k}$};
  \node[above=0.5mm of zp, font=\footnotesize, green!35!black] {$\zp_t \in \R^{k}$};
  \node[box, fill=gray!8, right=6mm of zc] (pred) {predictor\\$p_\phi$};
  \node[right=5mm of pred] (zhat) {$\hat{\z}_{t+1}$};
  \draw[->] (obs) -- (enc); \draw[->] (enc) -- (zp.west);
  \draw[->] (zc.east) -- (pred); \draw[->] (pred) -- (zhat);
  \node[below=1mm of pred, font=\footnotesize, align=center] {$\mathcal{L}_{\text{pred}}$ on\\the full latent};
  \node[note, below=7mm of zp, xshift=-8mm] (trip) {\textbf{triplet loss}: consecutive frames close, other trajectories apart};
  \node[note, below=7mm of zc, xshift=8mm] (sig) {\textbf{SIGReg}: content stays isotropic and does not collapse};
  \draw[->, green!40!black] (zp.south) -- (trip.north -| zp.south);
  \draw[->, gray!70!black] (zc.south) -- (sig.north -| zc.south);
  \begin{scope}[shift={($(zhat.east)+(1.6,-0.9)$)}]
    \draw[gray!60] (0,0) circle (0.75);
    \draw[->, gray!60] (-0.95,0) -- (0.95,0); \draw[->, gray!60] (0,-0.95) -- (0,0.95);
    \draw[very thick, green!45!black, ->] (200:0.75) arc (200:20:0.75);
    \draw[thick] (0,0) -- (55:0.75); \fill (55:0.75) circle (1.3pt);
    \draw (0.28,0) arc (0:55:0.28); \node[font=\footnotesize] at (27:0.47) {$\theta_t$};
    \node[font=\footnotesize, below=0.85cm, align=center] at (0,0) {progression angle\\read from $\zp_t$};
  \end{scope}
\end{tikzpicture}

\vspace{2mm}
\includegraphics[width=\linewidth]{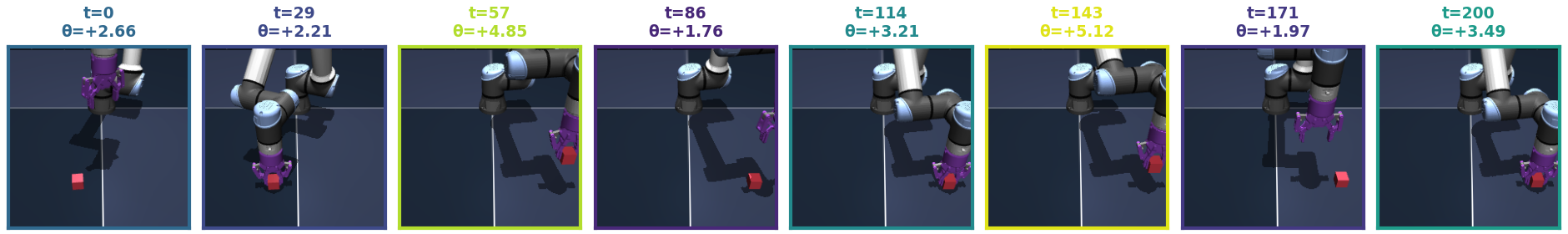}
\caption{\textbf{SD-JEPA in one picture.} \emph{Top:} the latent of a JEPA world model is split into a small progression block $\zp_t$ ($k \le 8$ of $D = 192$ dimensions) and a large content block $\zc_t$. A temporal triplet loss shapes the first, SIGReg regularises the second, and the predictor still sees the full latent. The angle $\theta_t$ of the progression block gives a direct reading of where the episode stands. \emph{Bottom:} $\theta_t$ along one OGBench-Cube pick-and-place episode ($k{=}8$). Frames are coloured by $\theta_t$; the angle moves between distinct ranges as the arm reaches, grasps, carries and releases the cube.}
\label{fig:teaser}
\end{figure}

\section{Introduction}
\label{sec:intro}

World models predict the consequences of actions in a compact latent space, so that an agent can plan by imagining futures instead of trying them. Joint-Embedding Predictive Architectures (JEPAs) \citep{lecun2022path,assran2023ijepa,bardes2024vjepa,Assran_Bardes_Fan_Garrido_Howes_Mojtaba_Komeili_Muckley_Rizvi_Roberts_et_al._2025} make these predictions in feature space rather than in pixels. \textsc{LeWM} \citep{maes2026lewm} recently showed that such a model can be trained end-to-end from pixels with only two losses, a next-embedding prediction loss and the SIGReg regulariser \citep{Balestriero_LeCun_2025}, and still plan competitively with a small model trained on a single GPU.

The latents of these models are good for prediction, and they are known to vary slowly and to straighten over time \citep{henaff2019straightening,sobal2022slow,maes2026lewm}. What they do not offer is a place where one can read \emph{where the agent is in its task}. If we want to know how far an episode has progressed, we have to fit a probe after training and hope the information is there. Yet progress is exactly the quantity a planner reasons about when it moves from a start state toward a goal, and it is one of the first things a person inspecting the model would want to see. Section~\ref{sec:motivation} explains why we expect it to fit in a few dimensions.

We argue that a world model should also hold a more global view of task completion, one that summarises how its states progress along an episode in a compact form. This paper proposes a simple way to build it in. SD-JEPA splits the latent of a JEPA world model into a small \emph{progression} block and a large \emph{content} block. A cosine-margin triplet loss \citep{dwibedi2019tcc,schneider2023cebra,thil2026steplearningstructuredembeddings} keeps consecutive frames close in the progression block and separates different trajectories, and SIGReg keeps regularising the content block only. Each loss sees only its own coordinates, so the two anti-collapse forces cannot cancel or duplicate each other on the same dimensions. The split adds no parameters. Throughout the paper, we use \emph{disentangling} in this architectural sense: progression and content are assigned separate coordinates by design, and we then test how much progression information is contained in each block.

We build SD-JEPA on \textsc{LeWM}, whose encoder, predictor and planner we keep unchanged. On its four control benchmarks, SD-JEPA plans on par with the original model (\S\ref{sec:planning}), so the extra structure does not cost planning performance. The progression block then becomes a useful readout (\S\ref{sec:reading}). Its angle $\theta_t$ follows progress within episodes and reverses when the agent backtracks. Its changes concentrate near contact events in manipulation, where latent prediction error tends to peak elsewhere. After a scene swap, it shows both a sharp jump and a lasting move to a new range, which a scalar prediction error cannot express.

Our contributions are the following.
\begin{enumerate}[leftmargin=*,topsep=2pt,itemsep=1pt]
\item \emph{A progression subspace for JEPA world models.} SD-JEPA reserves a few coordinates of the latent for progress, shapes them with a temporal triplet loss and keeps a distributional anti-collapse regulariser on the rest, with no extra parameters and no loss in planning performance.
\item \emph{Progress fits in a small subspace.} Between $1\%$ and $4\%$ of the latent is enough to carry within-episode task progress, and a linear readout of it predicts progress better than elapsed time on every benchmark we test.
\item \emph{Surprise anchored in the trajectory.} Changes in the progression angle concentrate around task events, and after a perturbation the angle moves to a new part of the progression. Surprise can then be read as where the agent now stands in its task.
\end{enumerate}

\section{Motivation for a progression subspace}
\label{sec:motivation}


Reserving a few coordinates for progress forces the model to learn a more compact and symbolic representation of state evolution, which we motivate on three lines of work.

\subsection{The quantities that matter are low-dimensional}

Deep representations typically lie on manifolds of much lower dimension than their width \citep{ansuini2019intrinsic}, and the factors that drive behaviour tend to vary slowly compared with raw observations \citep{wiskott2002slow}. \citet{bengio2017consciousness} proposed that high-level reasoning operates on a small, low-dimensional state selected from a much richer representation, an idea that echoes global-workspace theories in cognitive science \citep{baars1988cognitive,dehaene1998neuronal}. Recent interpretability work on language models finds a related structure in practice: the representations a model can report and reason with occupy only a narrow slice of its activations \citep{gurnee2026workspace}. The content of a scene needs a rich representation, but a summary of how far a task has advanced is a natural candidate for such a compact state.

\subsection{Models already encode time, but in no particular place}

Many high-level concepts are represented along linear directions of a network's activations \citep{park2024linear}, and time and space in particular can be linearly decoded from language models \citep{gurnee2024space}. If progress is already present but spread over many directions, giving it fixed coordinates makes it readable without a post-hoc probe. Fixed coordinates are also a way to supply structure explicitly. Unsupervised disentanglement is not identifiable without inductive biases or some supervision \citep{locatello2019challenging}, and the temporal order of frames is a weak signal that every trajectory dataset provides for free, yet the underlying generative process should be different than time \cite{thil2026steplearningstructuredembeddings}.

\subsection{Control problems are about moving toward a goal}

The benchmarks used for latent planning, and most systems studied in control, are trajectories from a start state toward a target. Representations built around temporal distance to a goal have proven useful for robot learning and goal-conditioned reinforcement learning \citep{sermanet2018tcn,ma2023vip,ma2023liv,wang2023quasimetric,park2024hilp,bae2024tldr}, and structured inductive biases of this kind are often argued to matter more than scale for higher-level cognition \citep{goyal2022inductive}. A model that represents its own state transitions in terms of progress toward the goal should also be easier to inspect: its latent then says not only what the scene looks like but where the episode is heading. A dedicated progression subspace brings this structure into a JEPA world model.

\section{Related work}
\label{sec:related}

\subsection{Latent world models}

Pixel-space world models such as DreamerV3 \citep{Hafner_Pasukonis_Ba_Lillicrap_2024} learn by reconstruction, while JEPAs predict in feature space \citep{lecun2022path,assran2023ijepa,bardes2024vjepa}. PLDM \citep{sobal2022slow}, DINO-WM \citep{zhou2024dino} and \textsc{LeWM} \citep{maes2026lewm} use such latents for planning, the last being trained end-to-end from pixels. Latent-action models \citep{garrido2026lawm,bruce2024genie,ye2025lapa} and large video world models \citep{team2025hunyuanworld,russell2025gaia,agarwal2025cosmos} explore other conditioning regimes. SD-JEPA keeps the encoder and predictor of \textsc{LeWM} unchanged and only restructures its latent.

\subsection{Progress and temporal-distance representations}

Ordering frames in time is a classic self-supervised signal, through time-contrastive networks \citep{sermanet2018tcn}, temporal cycle-consistency \citep{dwibedi2019tcc} and contrastive neural embeddings \citep{schneider2023cebra}. In robot learning, VIP and LIV learn embeddings whose distances reflect time to goal and use them as rewards \citep{ma2023vip,ma2023liv}. Quasimetric and Hilbert representations \citep{wang2023quasimetric,park2024hilp,bae2024tldr} make temporal distance a property of the latent geometry for goal-conditioned control. These methods shape the whole representation around temporal distance. We instead confine the temporal objective to a few coordinates and leave the rest of a JEPA latent untouched.

\subsection{Structured latents and temporal straightening}

Separating factors of variation is a long-standing goal of representation learning \citep{bengio2013representation}, but without inductive biases it is not identifiable \citep{locatello2019challenging}. Our split is a simple, explicit bias of this kind, with a weak supervisory signal (frame order) on one side only. The sample- and dimension-contrastive views of self-supervised learning \citep{garrido2023duality} explain why the triplet loss and SIGReg both prevent collapse, and why it matters that they act on different coordinates. \citet{wang2026straightening} add a curvature regulariser to latent trajectories and report better gradient-based planning. We tested the same regulariser on top of SIGReg and found that it hurts planning in our setting (App.~\ref{app:ablations}). Straightening can also serve out-of-distribution detection \citep{interno2025ai}.

\section{Background: \textsc{LeWM}}
\label{sec:background}

We follow the setting of \textsc{LeWM}. An offline dataset of trajectories $\{(o_{1:T_u}, a_{1:T_u})\}_{u=1}^{U}$ provides pixel observations $o_t$ and actions $a_t$. An encoder $f_\theta$ maps observations to latents $\z_t = f_\theta(o_t) \in \R^D$, and a predictor $p_\phi$ rolls dynamics forward in latent space, $\hat{\z}_{t+1} = p_\phi(\z_{t-H+1:t}, a_t)$, with history length $H$. \textsc{LeWM} trains both jointly with
\begin{equation}
\mathcal{L}_{\text{LeWM}} = \underbrace{\|\hat{\z}_{t+1} - \z_{t+1}\|_2^2}_{\mathcal{L}_{\text{pred}}} \;+\; \lambda \,\mathrm{SIGReg}(\z).
\label{eq:lewm}
\end{equation}
SIGReg pushes the marginal distribution of the latents toward an isotropic Gaussian, which prevents collapse. It projects a batch of latents $Z$ onto $M$ random unit directions $u^{(m)}$ and averages a univariate normality statistic $T$ (Epps--Pulley) over them:
\begin{equation}
\mathrm{SIGReg}(Z) = \frac{1}{M} \sum_{m=1}^{M} T\!\left(Zu^{(m)}\right).
\label{eq:sigreg}
\end{equation}
SIGReg constrains only the marginal $p(\z)$, not the dynamics $p(\z_{t+1}\mid\z_t)$. Planning uses the Cross-Entropy Method (CEM) to find actions whose predicted latent comes closest to the embedding of a goal image, $\z_g = f_\theta(o_g)$.

\section{SD-JEPA}
\label{sec:method}

\subsection{Splitting the latent}

We write the latent as two blocks,
\begin{equation}
\z_t = \bigl[\,\zp_t \,;\, \zc_t\,\bigr], \qquad \zp_t \in \R^k,\quad \zc_t \in \R^{D-k},
\end{equation}
so that the first $k$ coordinates form the \emph{progression subspace} and the remaining $D-k$ form the \emph{content subspace}. We use $D = 192$ as in \textsc{LeWM} and $k \in \{2, 4, 8\}$, that is, between $1\%$ and $4\%$ of the latent.

SD-JEPA keeps the prediction loss on the full latent and assigns one regulariser to each block:
\begin{equation}
\mathcal{L}_{\text{SD}} \;=\; \mathcal{L}_{\text{pred}}(\z)
\;+\; \lambda_S\,\mathrm{SIGReg}(\zc)
\;+\; \lambda_T\,\mathcal{L}_{\text{trip}}(\zp).
\label{eq:objective}
\end{equation}
SIGReg is unchanged except that its random directions are drawn in $\R^{D-k}$. The triplet loss works on the short training windows of \textsc{LeWM} (four consecutive frames). The anchor $a$ is the middle frame of a window, the positive $p$ is the next frame of the same trajectory, and the negative $n$ is the frame at the same position in another trajectory of the batch:
\begin{equation}
\mathcal{L}_{\text{trip}}(\zp) = \E_{(a,p,n)}\Bigl[\max\bigl(0,\; m - \cos(\zp_a, \zp_p) + \cos(\zp_a, \zp_n)\bigr)\Bigr],
\label{eq:triplet}
\end{equation}
with margin $m = 0.2$. The loss does not order frames explicitly: it asks consecutive frames of a trajectory to stay close in direction while different trajectories stay apart. A consistent progression along each episode is what we then observe in the trained model (\S\ref{sec:reading}). Triplets are built within the batch, so the loss adds no forward passes. Setting $\lambda_T = 0$ and applying SIGReg to the full $\z$ recovers \textsc{LeWM}. Hyperparameters are listed in App.~\ref{app:impl}.


SIGReg and the triplet loss both prevent collapse, and in the sample/dimension-contrastive view of \citet{garrido2023duality} they belong to the same family. Applied to the same coordinates, they would push on the same directions. The split avoids this by construction.

\begin{remark}[Separate gradient supports]
\label{rem:disjoint}
Since $\mathrm{SIGReg}(\zc)$ depends on $\z$ only through its last $D-k$ coordinates, its gradient with respect to $\z$ is zero on the first $k$. Likewise, the gradient of $\mathcal{L}_{\text{trip}}(\zp)$ is zero on the last $D-k$. The two regularisers therefore act on complementary sets of coordinates, and neither can compensate for the other.
\end{remark}

This is a direct consequence of the chain rule, and it is all we claim: the remark holds at the level of the latent. The prediction loss still acts on the full latent, and both regularisers reach the shared encoder weights, so interaction through the encoder is possible. Whether the split matters in practice is an empirical question, which we test in \S\ref{sec:split-ablation}. App.~\ref{app:theory} gives the general form with arbitrary orthogonal blocks and discusses when the split can fail.

\subsection{Reading and planning with the progression subspace}

From the first two progression coordinates we read an angle and a radius,
\begin{equation}
\theta_t = \mathrm{atan2}(\zp_{t,2},\, \zp_{t,1}), \qquad r_t = \|\zp_t\|_2 .
\end{equation}
The cosine triplet loss acts on directions, so the angle is the natural quantity to read. We use $\theta_t$ and its step-to-step change $|\Delta\theta_t|$ throughout \S\ref{sec:reading}.

For planning, we use the CEM planner of \textsc{LeWM} unchanged (300 candidate action sequences, 30 elites, 30 iterations on Push-T and 10 elsewhere). The default cost compares the content block of the predicted final latent with that of the goal, $\|\hat{\z}^{\mathrm{cont}}_H - \zc_g\|_2^2$. We also report the full-latent cost $\|\hat{\z}_H - \z_g\|_2^2$. We explored conditioning the predictor on $(\theta_t, r_t)$, an angular planning cost and an explicit straightening loss. None of them improved planning, so we report them in App.~\ref{app:ablations} and leave them out of the model.

\section{Planning with SD-JEPA}
\label{sec:planning}
\label{sec:experiments}

\subsection{Setup}

We use the four control benchmarks and the evaluation protocol of \citet{maes2026lewm}: Two-Room (2-D navigation), Reacher (DeepMind Control, \citealp{tassa2018deepmind}), Push-T (2-D manipulation) and OGBench-Cube (3-D manipulation, \citealp{park2024ogbench}). Episodes have a $50$-step horizon with goals $25$ steps ahead, and the planner uses a horizon of $5$ at frame-skip $5$. Every model trains for $10$ epochs at batch size $128$, the \textsc{LeWM} budget, with the \textsc{LeWM} encoder (ViT-tiny) and predictor ($6$-layer transformer), for $18$M parameters in total. We run three seeds ($0$, $42$, $3072$) and report the mean and sample standard deviation.

\begin{figure}[t]
\centering
\includegraphics[width=\linewidth]{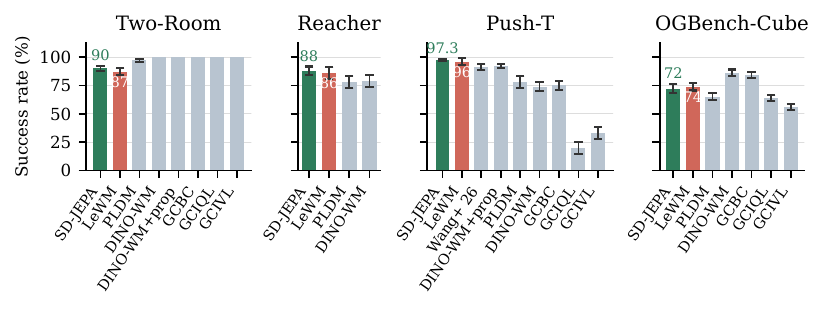}
\caption{\textbf{Planning success rate (\%) on the four \textsc{LeWM} benchmarks.} SD-JEPA (green) shows the three-seed mean and standard deviation at the best progression dimension $k$ for each environment. \textsc{LeWM} (red) and the other baselines (grey) are as reported by \citet{maes2026lewm}, plus the straightening method of \citet{wang2026straightening} on Push-T.}
\label{fig:headline-comparison}
\end{figure}

\subsection{SD-JEPA plans on par with \textsc{LeWM}}
\label{sec:main-results}

Table~\ref{tab:planning} compares SD-JEPA with \textsc{LeWM} at each progression dimension. At its best $k$, SD-JEPA is slightly above \textsc{LeWM} on Two-Room, Reacher and Push-T and slightly below on OGBench-Cube. All four differences are smaller than the seed-to-seed standard deviation of either method, so we read them as parity, not improvement. The point of this section is that the split does not cost planning performance. Figure~\ref{fig:headline-comparison} places both methods among the other baselines reported by \citet{maes2026lewm}. Like \textsc{LeWM}, SD-JEPA is behind the strongest baselines on Two-Room and OGBench-Cube and ahead of them on Push-T.\footnote{We report the published \textsc{LeWM} numbers. Our own rerun of \textsc{LeWM} with the public code reaches $76\%$ on Reacher instead of the published $86\%$; we traced the gap to differences between the released hyperparameter files and the values stated in the paper.}

\begin{table}[t]
\centering
\small
\begin{tabular}{lcccc}
\toprule
Method & Two-Room & Reacher & Push-T & OGB-Cube \\
\midrule
\textsc{LeWM} \citep{maes2026lewm} & $87 \pm 3.0$ & $86 \pm 5.0$ & $96 \pm 2.8$ & $74 \pm 3.5$ \\
\midrule
SD-JEPA, $k{=}2$ & $\mathbf{90.0} \pm 2.0$ & $84.0$ & $94.0$ & $72.0 \pm 4.0$ \\
SD-JEPA, $k{=}4$ & $88.0$ & $\mathbf{88.0} \pm 3.6$ & $96.0$ & $69.3$ \\
SD-JEPA, $k{=}8$ & $\mathbf{90.0}$ & $83.3$ & $\mathbf{97.3} \pm 1.2$ & $69.3$ \\
\bottomrule
\end{tabular}
\caption{Planning success rate (\%), three-seed mean, with the content-only planning cost. Standard deviations over training seeds are shown where available. \textsc{LeWM} numbers and standard deviations are taken from \citet{maes2026lewm}. Bold marks the best $k$ per environment.}
\label{tab:planning}
\label{tab:comparison}
\label{tab:cross-env-kprog}
\end{table}

\subsection{Choosing the progression dimension and the planning cost}

The best progression dimension depends on the environment (Table~\ref{tab:planning}). Push-T improves steadily with $k$, Reacher prefers $k{=}4$, OGBench-Cube prefers $k{=}2$, and Two-Room is flat. We do not find a single $k$ that is best everywhere, and selecting $k$ per environment is a cost of the method that we discuss in \S\ref{sec:limitations}. The progression block cannot replace the content block for planning: planning with the progression block alone reaches only $28\%$ on Push-T, since a few coordinates cannot describe a goal scene.

Using the full latent instead of the content block in the planning cost gives $3.3$ more points on Reacher at $k{=}8$. We did not compare the two costs systematically across environments and leave that to future work.

\subsection{Where the triplet loss is applied matters}
\label{sec:split-ablation}
\label{sec:subspace-falsifier}

To check that the split itself matters, and not just the extra loss, we compare four variants on Push-T (Table~\ref{tab:subspace-ablation}). Applying the triplet loss to the whole latent, with no split, gives the same result as \textsc{LeWM}. Keeping the split but still applying the triplet loss to the whole latent, so that it also acts on the content block, lowers success below \textsc{LeWM}. Only the version that confines the triplet loss to the progression block reaches or exceeds \textsc{LeWM}. The pattern is consistent with Remark~\ref{rem:disjoint}: the triplet loss helps when it stays off the content coordinates.

These runs use a single seed, and their differences are close to the seed-to-seed variation reported above (the three-seed mean of the $k{=}2$ configuration is $94\%$, against $98\%$ for this seed). We therefore treat the ablation as indicative.

\begin{table}[t]
\centering
\small
\begin{tabular}{lccc}
\toprule
Variant & SIGReg on & Triplet on & Push-T (\%) \\
\midrule
\textsc{LeWM} (our run) & $\z$ & none & $96$ \\
No split & $\z$ & $\z$ & $96$ \\
Split, triplet on full latent & $\zc$ & $\z$ & $92$ \\
\textbf{SD-JEPA} ($k{=}2$) & $\zc$ & $\zp$ & $\mathbf{98}$ \\
\bottomrule
\end{tabular}
\caption{Where each loss is applied, on Push-T (single seed, $3072$).}
\label{tab:subspace-ablation}
\end{table}

\section{Reading the progression subspace}
\label{sec:reading}
\label{sec:latent-diagnostics}

The previous section shows that the split costs nothing in planning. This section asks what the progression subspace gives in return. We look at four questions in turn: how the two blocks differ, whether $\theta_t$ tracks progress, whether it marks task events, and how it reacts to surprise.

\subsection{The two blocks organise episodes differently}

If the split works as intended, the progression block should look alike across episodes of the same task, while the content block should separate them. This is what we observe on Push-T (Fig.~\ref{fig:latent-geometry-body}, App.~\ref{app:latent-geometry}). In a t-SNE projection, content embeddings form one cluster per episode, while progression embeddings form arcs that overlap across episodes. At a fixed time step, progression embeddings of different episodes have a mean cosine similarity of about $0.5$, against about $0$ for content embeddings (Table~\ref{tab:cosine-diagnostic}).

This shows that the two blocks are organised differently. It does not show that the content block is free of progression information.

\subsection{Does $\theta_t$ track progress?}
\label{sec:probe}

Since the triplet loss is built from consecutive frames, $\theta_t$ could simply learn elapsed time. We therefore compare its rank correlation with the time step against its correlation with task-specific quantities: the distance between the end effector and the target on Reacher, the block-to-target distance on OGBench-Cube, and the block angle error on Push-T (App.~\ref{app:latent-geometry}). The answer depends on the environment. On OGBench-Cube, the task quantity correlates more strongly with $\theta_t$ than time does. On Reacher and on Push-T at $k{=}8$ the two are about equal. On Two-Room and on Push-T at $k{=}2$, time dominates. So $\theta_t$ is partly a clock and partly a task-phase coordinate, and it becomes less clock-like as $k$ grows. Individual episodes show where the two differ: when a Push-T agent reverses, $\theta_t$ reverses with it, which a clock cannot do.

We then fit linear probes that predict a normalised task-progress signal from different features, within each held-out episode ($40$ episodes per environment, leave-one-step-out cross-validation; Table~\ref{tab:probe-crossenv}). On all four environments, the $8$ progression coordinates predict progress better than the time step, than $(\sin\theta_t, \cos\theta_t)$, and than a random $2$-D projection of the full latent. The gap over the time step is largest on Reacher and on OGBench-Cube. On Two-Room the advantage over time is small, consistent with the correlations above.

The probe has a clear limit. Pooled across episodes, with different starting distances, only the time step predicts progress reliably (App.~\ref{app:phase-align}). $\theta_t$ is a phase coordinate within a trajectory, not a calibrated distance shared across episodes.

\begin{table}[t]
\centering
\small
\begin{tabular}{lcccc}
\toprule
Feature (dimension) & Cube & Push-T & Reacher & Two-Room \\
\midrule
Time step (1) & $0.29$ & $0.62$ & $0.29$ & $0.69$ \\
Random projection of $\z$ (2) & $0.26$ & $0.30$ & $0.24$ & $-0.27$ \\
$(\sin\theta_t, \cos\theta_t)$ (2) & $0.56$ & $0.42$ & $0.34$ & $0.04$ \\
$\zp$ (8) & $\mathbf{0.91}$ & $\mathbf{0.91}$ & $\mathbf{0.95}$ & $\mathbf{0.72}$ \\
\bottomrule
\end{tabular}
\caption{Within-episode linear probe of task progress: mean $R^2$ over $40$ held-out episodes per environment ($k{=}8$ checkpoints). The full distribution and the pooled probe are in App.~\ref{app:phase-align}.}
\label{tab:probe-crossenv}
\end{table}

\subsection{Does $\theta_t$ mark task events?}

The standard surprise signal in latent world models is the prediction error $\|\hat\z_{t+1}-\z_{t+1}\|$, which we call z-MSE. It says when the model was wrong, but not about what. On OGBench-Cube, the moments when the gripper closes or opens give an objective set of task events, about four per pick-and-place episode. We ask whether $|\Delta\theta_t|$ or z-MSE is more concentrated near these events, labelling every step within a tolerance of an event as positive and computing the AUROC of each signal over $40$ held-out episodes (Table~\ref{tab:phase-auroc}).

$|\Delta\theta_t|$ scores higher than z-MSE at every tolerance and on most individual episodes. However, at the two tightest tolerances, both signals score below $0.5$. Neither one peaks exactly on the event step: z-MSE peaks while the cube is carried, between events, and $|\Delta\theta_t|$ peaks close to events but often a few steps off. Only at $\pm 3$ steps does $|\Delta\theta_t|$ rise above chance. The fair reading is that changes in $\theta_t$ sit closer to task events than prediction error does, not that $\theta_t$ is a precise event detector. Figure~\ref{fig:phase-overlay-cube} shows a typical episode.

Two control tests mark the limits of this result (App.~\ref{app:phase-align}). When we corrupt a stretch of actions, which creates an anomaly in magnitude rather than a task event, z-MSE finds it better than $|\Delta\theta_t|$. When we segment episodes with a change-point detector, z-MSE gives the more precise boundaries at tight tolerance, and the angle trace only wins at the loosest tolerance. The two signals answer different questions, and we see $\theta_t$ as a complement to prediction error, not a replacement.

\begin{table}[t]
\centering
\small
\begin{tabular}{lccc}
\toprule
Tolerance & z-MSE & $|\Delta\theta_t|$ & Episodes where $|\Delta\theta_t|$ is higher \\
\midrule
$\pm 1$ step  & $0.24$ & $0.41$ & $39 / 40$ \\
$\pm 2$ steps & $0.36$ & $0.47$ & $34 / 40$ \\
$\pm 3$ steps & $0.51$ & $\mathbf{0.57}$ & $29 / 40$ \\
\bottomrule
\end{tabular}
\caption{Pooled AUROC of each signal for steps near a gripper-contact event on OGBench-Cube ($40$ held-out episodes, $160$ events, $k{=}8$). $0.5$ is chance level: at the two tightest tolerances, both signals peak mostly away from the exact event step.}
\label{tab:phase-auroc}
\end{table}

\begin{figure}[t]
\centering
\includegraphics[width=0.6\linewidth]{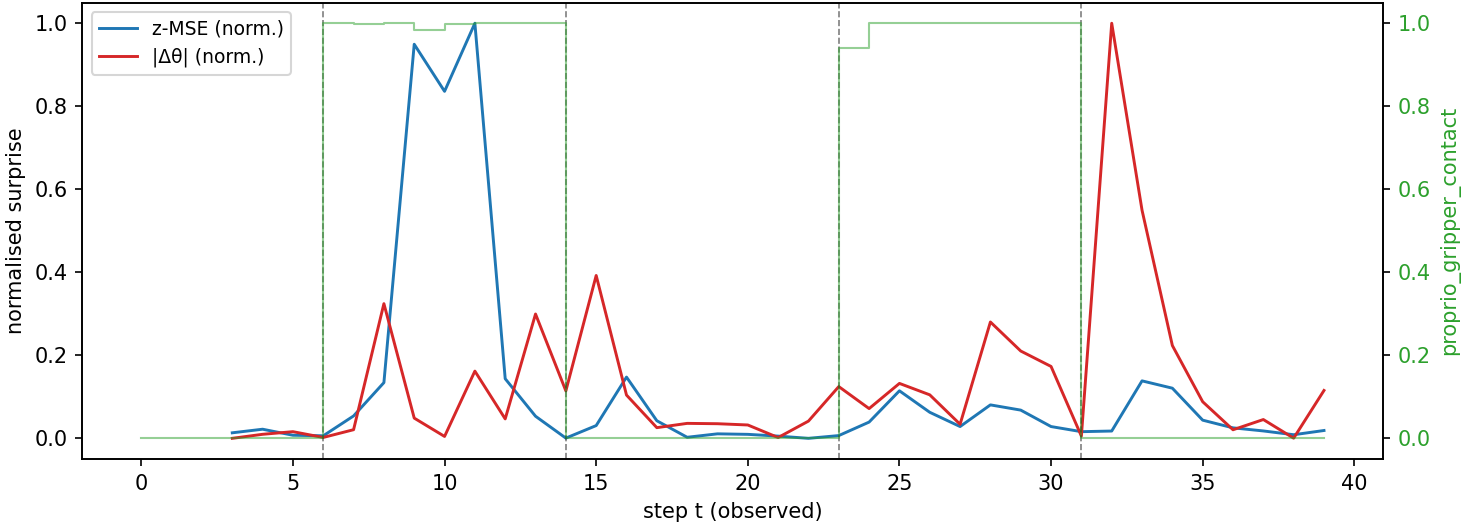}
\caption{One OGBench-Cube episode. Normalised z-MSE (blue) and $|\Delta\theta_t|$ (red) against gripper contact (green; dashes mark open/close events). z-MSE peaks early, while the cube is being manipulated; $|\Delta\theta_t|$ peaks near the contact events.}
\label{fig:phase-overlay-cube}
\end{figure}

\subsection{How does $\theta_t$ react to surprise?}

A scalar surprise signal can only rise and fall. Because $\theta_t$ is a position on a circle, it can also \emph{stay} somewhere new. We test this on Push-T with a scene swap: at step $50$ of a held-out episode, the next observation is replaced by a frame from another episode in which the T-block sits on the opposite side of the goal, and the planner then continues from there. We picked this pair of episodes because their solving modes differ clearly.

$\theta_t$ shows two distinct responses (Fig.~\ref{fig:teleport-continue}). At the swap, $|\Delta\theta_t|$ spikes for a single frame, much like a prediction-error spike. This marks the \emph{moment} of surprise. After the spike, $\theta_t$ does not return to its previous range. It settles in a different part of the circle, corresponding to the new solving mode, which indicates \emph{what} changed. Prediction error shows only the first response. Three other perturbations, namely splicing in a different episode, inserting a single foreign frame and reversing a short window of frames, are reported in App.~\ref{app:perturb}. A single foreign frame produces a spike followed by a quick return. Reversing frames, which leaves the set of images unchanged, still produces spikes at both ends of the reversed window, so the signal responds to violated dynamics and not only to new images. This is a qualitative analysis on a handful of episodes. The behaviour loosely recalls how head-direction cells re-anchor after the environment changes \citep{taube1990hd,knierim2020hippocampus}, but we make no mechanistic claim.

\begin{figure}[t]
\centering
\begin{subfigure}[b]{0.495\linewidth}
\centering
\includegraphics[width=\linewidth]{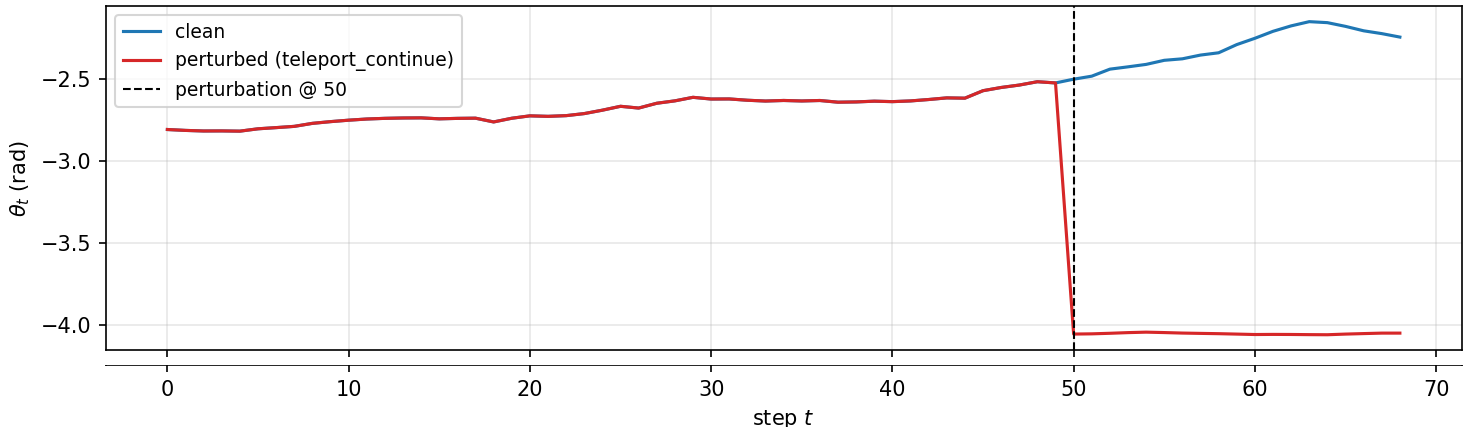}
\subcaption{Progression angle $\theta_t$: it moves to a new range.}
\label{fig:teleport-continue-prog}
\end{subfigure}\hfill
\begin{subfigure}[b]{0.495\linewidth}
\centering
\includegraphics[width=\linewidth]{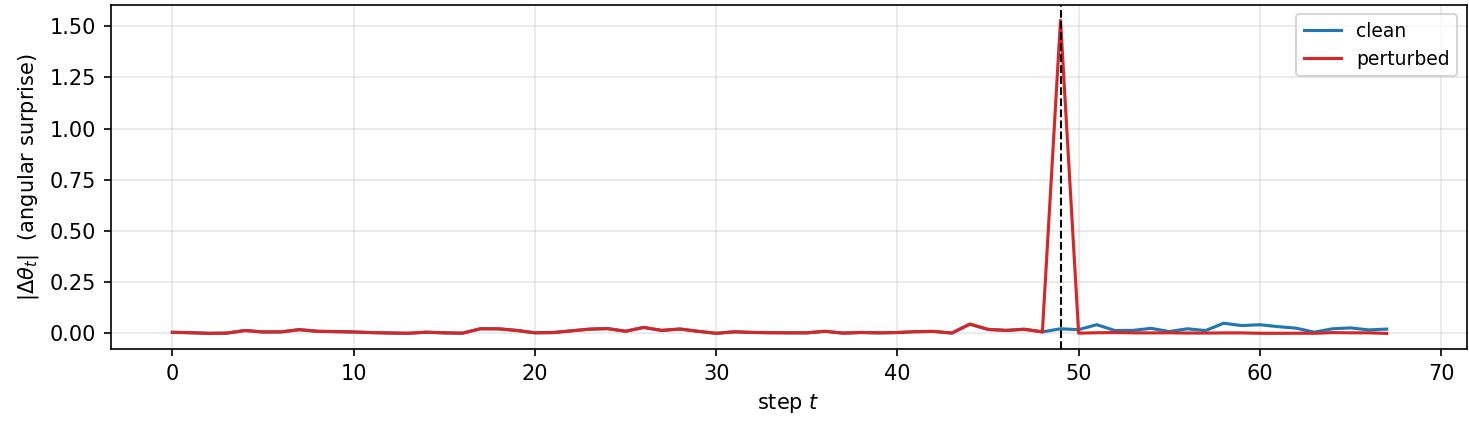}
\subcaption{Angular change $|\Delta\theta_t|$: a single-frame spike.}
\label{fig:teleport-continue-surprise}
\end{subfigure}
\caption{Push-T, scene swap at step $50$ followed by continued planning ($k{=}2$). \textbf{(a)} After the swap, $\theta_t$ (unwrapped) does not return to its earlier trace but settles in a new range. \textbf{(b)} $|\Delta\theta_t|$ spikes at the swap only.}
\label{fig:teleport-continue}
\end{figure}

\section{Limitations}
\label{sec:limitations}
\label{sec:cross-env-summary}
\label{sec:discussion}

SD-JEPA matches the planning performance of the model it extends but does not clearly improve on it, and the best progression dimension has to be chosen per environment. All experiments use four short-horizon benchmarks with a $10$-epoch training budget. Whether a progression subspace helps more on long-horizon or multi-stage tasks, where knowing the current phase matters more, is the most important open question.

The progression angle is not a clean task coordinate everywhere. On some environments it stays close to elapsed time, and it is not calibrated across episodes. We also show that the two blocks are organised differently and that the progression block carries progress information, but not yet that the content block lacks it. 

Finally, the split is fixed and the readout is angular. A learned decomposition or a non-angular progression head may work better, and other JEPA backbones remain to be tested.

\section{Conclusion}
\label{sec:conclusion}

SD-JEPA reserves a small subspace of a JEPA latent for task progression, shapes it with a temporal triplet loss, and leaves the remaining content coordinates to SIGReg. The split costs neither parameters nor planning performance on the \textsc{LeWM} benchmarks. In exchange, the latent gains a coordinate that can be read directly: it follows progress within an episode, changes near task events, and responds to surprise in a way that separates when something happened from what changed. We see this as a step toward world models whose internal state is structured around the goal they pursue, and we expect progression coordinates to be most useful on longer, multi-stage tasks.

{\small
\bibliographystyle{plainnat}
\bibliography{refs}
}

\appendix

\section{Extended theoretical analysis}
\label{app:theory}

This appendix reproduces the full theoretical content summarised in \S\ref{sec:method}, including the corollary on the marginal product structure, the Garrido-style duality discussion, and three further results (consistency with emergent straightening, identifiability of the progression axis, and conditions under which the split fails).

\subsection{Non-conflict under the subspace split}

\begin{proposition}[Disjoint gradient supports]
\label{prop:disjoint}

Let $\z \in \R^D$ be written as $\z = P\zp + Q\zc$ with $P, Q$ as in \S\ref{sec:method}, and let $\mathrm{SIGReg}$ be applied on $\zc$ and $\mathcal{L}$ on $\zp$. Then for every sample,
\begin{equation}
\mathrm{span}\!\left(\nabla_{\z}\,\mathrm{SIGReg}(Q^\top \z)\right) \;\subseteq\; \mathrm{col}(Q),
\qquad
\mathrm{span}\!\left(\nabla_{\z}\,\mathcal{L}(P^\top \z)\right) \;\subseteq\; \mathrm{col}(P),
\end{equation}
and these two subspaces are orthogonal.
\end{proposition}
\begin{proof}[Proof sketch]
$\mathrm{SIGReg}(Q^\top\z)$ depends on $\z$ only through $Q^\top\z$, so its gradient with respect to $\z$ factors through $Q$: $\nabla_{\z}\, \mathrm{SIGReg}(Q^\top\z) = Q \,\nabla_{\zc}\mathrm{SIGReg}(\zc)$. Identically for the second loss $\mathcal{L}$ on $\zp$ and $P$. Orthogonality $P^\top Q = 0$ gives the conclusion.
\end{proof}
Proposition~\ref{prop:disjoint} formalises the informal claim that an arbitrary loss on $\zp$ alone does not interfere with SIGReg on $\zc$. The triplet loss on $\zp$ and the prediction loss on the full $\z$ do have components in $\mathrm{col}(Q)$ (through the content subspace of the predictor output), so the non-conflict is specifically between the marginal-distributional term on $\zc$ and any progression-geometric term on $\zp$.

\begin{corollary}
Under the subspace split, the isotropic Gaussian constraint imposed by $\mathrm{SIGReg}$ holds on the marginal $p(\zc)$ regardless of the geometry imposed on $p(\zp)$ by any loss acting on $\zp$ alone. 
\end{corollary}

\subsection{Relation to the sample/dimension-contrastive duality}

The cosine-margin triplet loss on $\zp$ is a sample-contrastive criterion in the sense of \citet{garrido2023duality}: it repels $\zp_a$ from $\zp_n$ while attracting it to $\zp_p$, in a space of $L_2$-normalized embeddings (cosine similarity is the inner product of unit vectors). SIGReg on $\zc$ is not exactly a dimension-contrastive criterion, but it belongs to the same \emph{variance-repulsion family}: it penalizes any departure of the marginal from isotropic Gaussian, and in particular penalizes any direction along which the variance is abnormally low (collapse) or the covariance is nonzero (feature correlation). Under the double-$L_2$-normalization assumption of \citet{garrido2023duality}, the anti-collapse force of the triplet loss on $\zp$ equals (up to additive constants) that of a VICReg-style dimension-contrastive criterion. SIGReg subsumes that variance/covariance criterion as a special case (matching the second moments of an isotropic Gaussian).

\begin{proposition}[No double-counting on orthogonal subspaces]
\label{prop:noduplication}

Let $\mathcal{L}_{\text{trip}}$ act on $\zp$ and $\mathrm{SIGReg}$ act on $\zc$. Under the subspace split, their anti-collapse gradients act on orthogonal coordinates and therefore do not duplicate: removing either term cannot be compensated by increasing the other, since the other has no gradient in the affected coordinate directions.
\end{proposition}
\begin{proof}[Proof sketch]
Apply Proposition~\ref{prop:disjoint} with the triplet loss as the function $\mathcal{L}$ on $\zp$.
\end{proof}
Proposition~\ref{prop:noduplication} is the practical consequence of the duality: even though triplet and SIGReg are \emph{of the same family} (in the Garrido sense), on the subspace split they are \emph{complementary rather than redundant}. This gives a clean theoretical justification for using both terms simultaneously rather than treating one as subsuming the other.

\subsection{Consistency with emergent straightening}
\label{sec:straight-theory}

\textsc{LeWM} reports (Sec.~5 and Appendix~H of \citep{maes2026lewm}) that as training proceeds, the cosine similarity between consecutive latent velocity vectors $\vv_t = \z_{t+1} - \z_t$ increases emergently, without any explicit regularization. This perceptual straightening property, inspired by \citet{henaff2019straightening}, is an invariant of \emph{good} representations of natural dynamics: straight paths in latent space are linearly extrapolable, and so compound less error over multi-step rollouts.

Our straightening term $\mathcal{L}_{\text{straight}}(\z) = -\E[\cos(\vv_t, \vv_{t+1})]$ makes this trend explicit. It is:
\begin{itemize}[leftmargin=*]
\item \emph{scale-invariant} in $\z$: cosine depends only on directions of velocity, not magnitudes; this means $\mathcal{L}_{\text{straight}}$ does not compete with the variance-related component of $\mathrm{SIGReg}$;
\item \emph{subspace-consistent}: applying $\mathcal{L}_{\text{straight}}$ on the full $\z$ does not prefer one of $\zp,\zc$ over the other; in practice we may apply it only on $\zc$ if $\zp$ is already sufficiently constrained by the triplet ordering;
\item \emph{a soft regularizer}: it does not force $\cos=1$ (perfectly straight dynamics), only moves the population mean upward.
\end{itemize}
\begin{proposition}[Soft bias towards the emergent optimum]
Let $\mathcal{L}_{\mathrm{LeWM}}$ denote the \textsc{LeWM} base loss and assume (as empirically reported) that along its optimization path $\E[\cos(\vv_t, \vv_{t+1})]$ increases monotonically toward a limit $c^\star < 1$. Adding $\lambda_{\mathrm{str}}\,\mathcal{L}_{\mathrm{straight}}$ as a soft term (with $\lambda_{\mathrm{str}} > 0$) preserves the limiting direction of optimization: the gradient of $\mathcal{L}_{\mathrm{straight}}$ at the LeWM optimum is aligned with the direction of increasing straightness, hence the added term is non-antagonistic.
\end{proposition}
Empirically this is the weakest of our claims, but it rests on the reported emergence from previous works. It motivates treating straightening as a warm-start regularizer rather than a hard constraint.

\subsection{Identifiability of the progression axis}
Under the coordinate split with $k = 2$, the angular coordinate $\theta_t = \mathrm{atan2}(\zp_2, \zp_1)$ is identifiable up to a global rotation (the triplet loss depends only on cosine similarities, so it does not pin a unique reference direction). This rotation freedom is harmless when the planning cost itself is rotation-invariant (the default $\zc$-MSE cost is, by construction).

The identifiability is local to episodes: episodes with different durations or solving modes can each present a coherent angular progression in $\zp$ without any global reference direction shared across them.

\subsection{When does the subspace split fail?}
The construction above assumes the progression subspace $\zp$ is a reasonable $k$-dimensional summary of the episode phase. This assumption fails when:
\begin{enumerate}[leftmargin=*]
\item The environment lacks coherent episodic temporal structure (e.g.\ purely exploratory rollouts with no consistent task progression). In that case the triplet loss has no useful temporal signal and the split degenerates to the LeWM baseline.
\item The task has multiple branching progressions (e.g.\ different goals per episode with no shared notion of completion). The triplet loss can still impose local temporal coherence within each branch, but a single $k$-dimensional progression manifold may not accommodate the branching cleanly; goal-conditioned variants are a natural extension.
\item $k$ is too large: the progression subspace absorbs content information, polluting the $(\theta, r)$ reading. We sweep $k \in \{2, 4, 8\}$ as an ablation.
\end{enumerate}

\section{Implementation notes}
\label{app:impl-extended}

We build on the public reference implementation of \textsc{LeWM} \citep{maes2026lewm} for the encoder, predictor, action-conditioning, and SIGReg modules. SD-JEPA additions are localised to (a) splitting the projected latent into its two blocks, (b) instantiating the cosine-margin triplet on $\zp$, and (c) the optional polar-conditioning / angular-planning-cost paths described in App.~\ref{app:ablations}. The triplet uses within-batch sampling on the four-frame training windows: the anchor is the middle frame, the positive is the next frame of the same window, and the negative is the anchor frame of another window in the batch (a random non-zero roll of the batch). It reuses the embeddings already computed for the prediction loss, so it adds no forward passes.

\section{Hyperparameter summary}
\label{app:impl}

\begin{table}[h]
\centering
\small
\begin{tabular}{lp{0.72\linewidth}}
\toprule
Name             & Meaning \\
\midrule
\textsc{LeWM}    & Baseline. No split, SIGReg on the full latent, no triplet. \\
SD-JEPA          & The proposed configuration: subspace split with $k_{\mathrm{prog}}{=}2$, SIGReg restricted to $\zc$, cosine triplet on $\zp$, no other auxiliary terms. The optimal $k_{\mathrm{prog}}$ varies per environment ($k\in\{2, 4, 8\}$); we write SD-JEPA ($k{=}\cdot$) for the sweep. \\
No split         & Control without the split. $k_{\mathrm{prog}}{=}0$ (no split), SIGReg and triplet both act on the full $\z$. \\
Split, full triplet & Split present ($k_{\mathrm{prog}}{=}2$) but the triplet is applied on the full $\z$ instead of $\zp$. \\
\bottomrule
\end{tabular}
\caption{Configurations compared in \S\ref{sec:split-ablation}. The auxiliary variants (adding straightening, polar predictor conditioning, and the angular planning cost) are described in App.~\ref{app:ablations}.}
\label{tab:rung-legend}
\end{table}

Per-environment configs, dataset paths, hyperparameter sweeps, and training-time stop conditions are documented in the released code. The base hyperparameters: AdamW, lr $5\!\times\!10^{-5}$, weight decay $10^{-3}$, cosine-annealed over the 10 training epochs, batch size $128$, image resolution $224 \times 224$, ViT-tiny encoder (patch size $14$), $6$-layer transformer predictor with AdaLN-zero action conditioning, embedding dimension $192$. The frame-skip is $5$ for all environments. History length $3$ with one predicted step, so training windows have four frames. Loss weights: $\lambda_S = 0.09$ (SIGReg, $1024$ random projections, $17$ knots), $\lambda_T = 0.10$ (triplet), triplet margin $m = 0.2$, positive offset of one frame. Gradient clipping at $1.0$, bf16 precision. Optional auxiliary weights ($\lambda_{\text{str}}$ for straightening, $\gamma_\theta, \delta_r$ for the angular planning cost) are $0$ in SD-JEPA; see App.~\ref{app:ablations}.

\section{Additional auxiliary-term explorations}
\label{app:ablations}

While exploring the design space of SD-JEPA we tested several optional auxiliary terms that, in our setup, do not improve SD-JEPA. We document them here for completeness and as a reference for follow-up work. (i)~\textbf{Straightening} adds the temporal-straightening regulariser $\mathcal{L}_{\text{straight}} = -\E[\cos(\vv_t, \vv_{t+1})]$ on the full latent. On Push-T this variant is roughly $14$ points below SD-JEPA. Applied to the full $\z$, the straightening cost acts on the same coordinates as SIGReg, which is the kind of overlap the split is designed to avoid. The same regulariser is the central contribution of \citet{wang2026straightening}, who use it under stop-gradient anti-collapse where the conflict does not arise (Sec.~\ref{sec:related}). (ii)~\textbf{Polar conditioning} adds $(\sin\theta, \cos\theta, r)$ to the predictor's AdaLN conditioning on top of straightening and yields a further ${\sim}2$-point drop, suggesting that exposing the angular readout to the predictor at training time over-specialises the content channel without compensating gains. (iii)~\textbf{Angular planning cost} additionally uses the angular planning cost $\gamma_\theta\,(1 - \cos(\hat\theta - \theta_g)) + \delta_r\,(\hat r - r_g)^2$ on $\zp$ at evaluation time. We also tested an auxiliary first-anchor regulariser pulling the genuine first-of-episode frames in each batch toward a fixed reference direction in $\zp$, as an attempt to globally calibrate $\theta_t$ across episodes; this variant matches SD-JEPA and supports re-engaging the angular planning cost without regression but does not improve on it in our setup. We document these explorations to show that SD-JEPA does not need any of these extra regularisers, and that adding them tends to hurt rather than help.

Additional environments can be further explored such as WorldGym \cite{quevedo2506worldgym}, richer learning approaches directly in the latent \cite{grill2020bootstrap}, or totally different modalities such as in neuroscience \cite{dong2024brain} to observe if we can link such progression metrics to isolate distinct phenomena. 

\section{Extended latent geometry across environments}
\label{app:latent-geometry}

This appendix collects the per-environment latent-geometry figures complementing \S\ref{sec:latent-diagnostics}. For each environment we ran the latent-compass analysis on the trained checkpoint indicated, on a small set of held-out episodes spanning the dataset's full episode-id range. Per checkpoint we report (i) the t-SNE projections of $\zp$ and $\zc$, (ii) the trajectory-overlay heatmap of $\theta_t$ in raw state space, (iii) per-episode signed-Spearman heatmaps against candidate task-progress proxies, and (iv) one frame strip per env exhibiting an interesting non-monotone or multi-phase $\theta$ behaviour.

\begin{table}[h]
\centering
\small
\begin{tabular}{lccc}
\toprule
Run                                     & $|\rho|$ vs.\ clock & best non-clock proxy ($|\rho|$)         & clock vs.\ task           \\
\midrule
Push-T SD-JEPA ($k_{\mathrm{prog}}{=}2$)      & \textbf{0.91}       & \texttt{block\_angle\_err} (0.81)       & clock dominates          \\
Push-T SD-JEPA ($k{=}2$) ($k_{\mathrm{prog}}{=}8$)      & 0.56                & \texttt{block\_angle\_err} (0.57)       & tied                     \\
Reacher ($k_{\mathrm{prog}}{=}8$)         & 0.53                & \texttt{ee\_target\_dist} (0.49)        & tied                     \\
Two-Room ($k_{\mathrm{prog}}{=}8$)        & 0.50                & \texttt{agent\_target\_dist} (0.34)     & clock leads\,$^\dagger$  \\
\textbf{Cube ($k_{\mathrm{prog}}{=}8$)}   & 0.47                & \textbf{\texttt{block\_target\_dist} (0.59)} & \textbf{task wins}  \\
\bottomrule
\end{tabular}
\caption{Per-environment mean $|\rho|$ between $\theta_t$ and the elapsed-time clock vs.\ the strongest non-clock proxy. $^\dagger$ Two-Room's expert demos include out-and-back trajectories where $\theta$ traces a U-shape; Spearman is monotonic and underestimates non-monotonic phase coupling (Fig.~\ref{fig:tworoom-frame-strip}). Cube is the only environment where a task-physical signal beats the clock.}
\label{tab:theta-corr}
\end{table}

\subsection{Push-T at $k_{\mathrm{prog}}{=}2$ }

\begin{figure}[h]
\centering
\begin{minipage}[b]{0.36\linewidth}
\centering
\includegraphics[width=\linewidth]{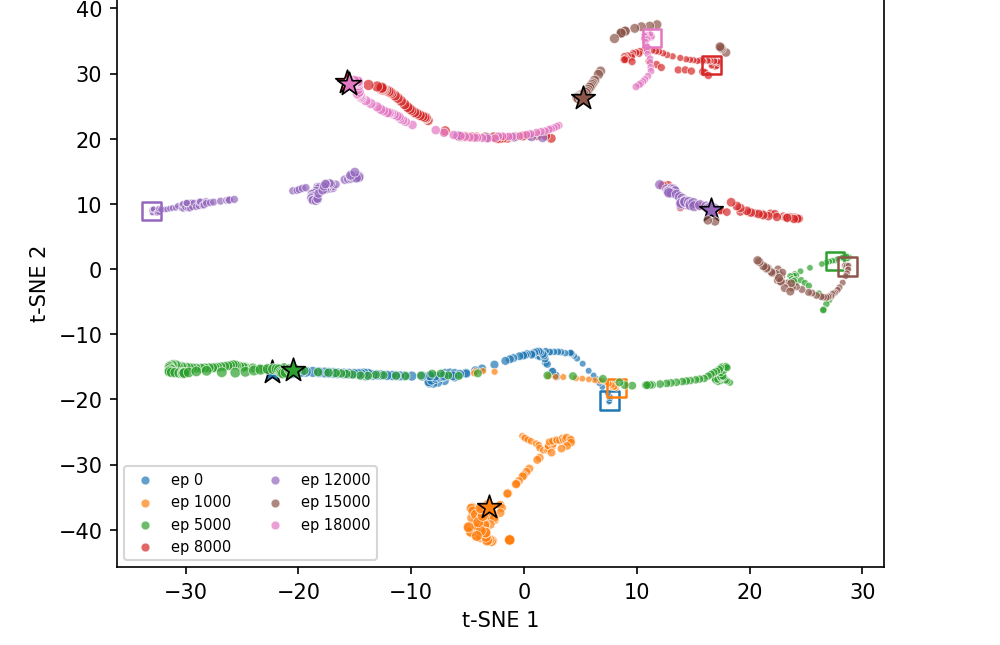}
\subcaption{$\zp$: per-episode arcs partially overlap, indicating shared cross-episode progression structure.}
\label{fig:tsne_zprog}
\end{minipage}
\hfill
\begin{minipage}[b]{0.36\linewidth}
\centering
\includegraphics[width=\linewidth]{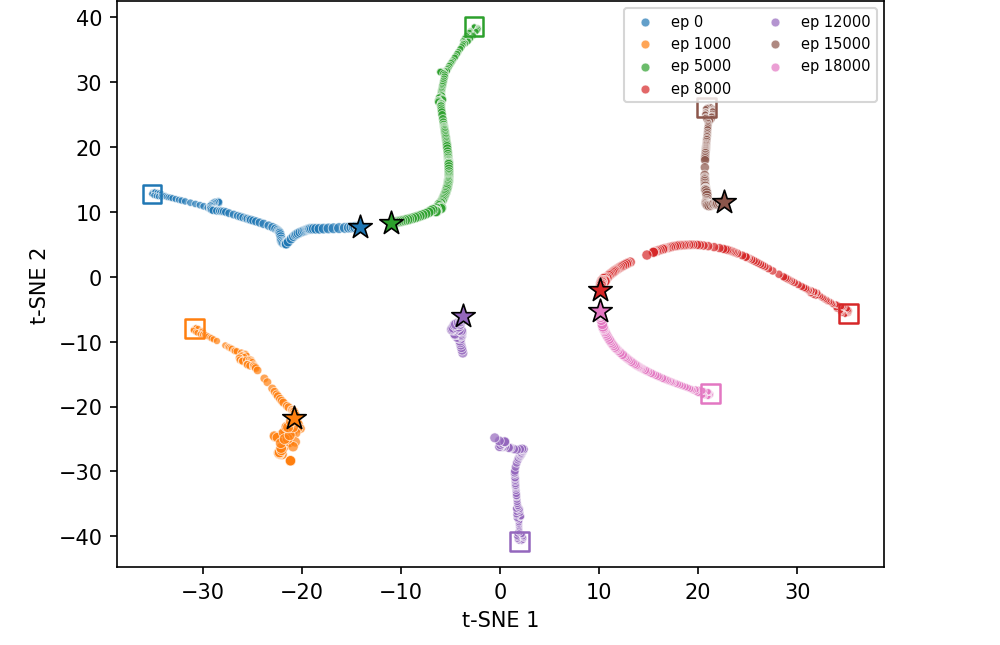}
\subcaption{$\zc$: each episode forms its own distinct cluster, indicating episode-specific content.}
\label{fig:tsne_zcont}
\end{minipage}
\hfill
\begin{minipage}[b]{0.22\linewidth}
\centering
\includegraphics[width=\linewidth]{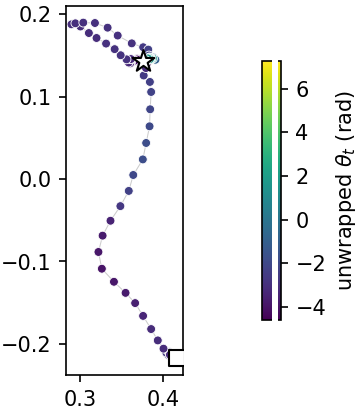}
\subcaption{Cube ep.\ 1000 ($k{=}8$, seed 42): planar trajectory coloured by $\theta_t$.}
\label{fig:cube-state-traj}
\end{minipage}
\caption{Latent geometry of the trained Push-T SD-JEPA ($k{=}2$) (a, b; $k_{\mathrm{prog}}{=}2$, seed 3072, epoch 10, seven held-out episodes; squares: start, stars: end) and a cube state-traj overlay (c). \textbf{(a, b)} The disjoint-supports prediction of Prop.~\ref{prop:disjoint} -- cross-episode mixing on the progression side, per-episode separation on the content side -- is visible in both panels. \textbf{(c)} Cube is the only env where $\theta_t$ correlates more strongly with a task-physical signal (\texttt{block\_target\_dist}, $|\rho|{=}0.59$) than with the elapsed-time clock ($0.47$); the angular gradient along the planar path makes the compass visible.}
\label{fig:latent-geometry-body}
\end{figure}

The body's t-SNE pair (Fig.~\ref{fig:latent-geometry-body}) is from this checkpoint. The Push-T $r$-coordinate stays in roughly $[0.90, 1.57]$ with per-episode std up to $0.49$, $\theta$-span around $3.3$ rad. Per-episode signed Spearman against five candidate proxies (Fig.~\ref{fig:theta-corr-pusht}) shows that $|\rho|$ with \texttt{step\_idx} is consistently very high ($\approx 0.86$--$1.00$) on all but one episode (ep 8000, the regression case), with \texttt{block\_angle\_err} and \texttt{block\_goal\_dist} as close runners-up.

\begin{figure}[h]
\centering
\includegraphics[width=0.7\linewidth]{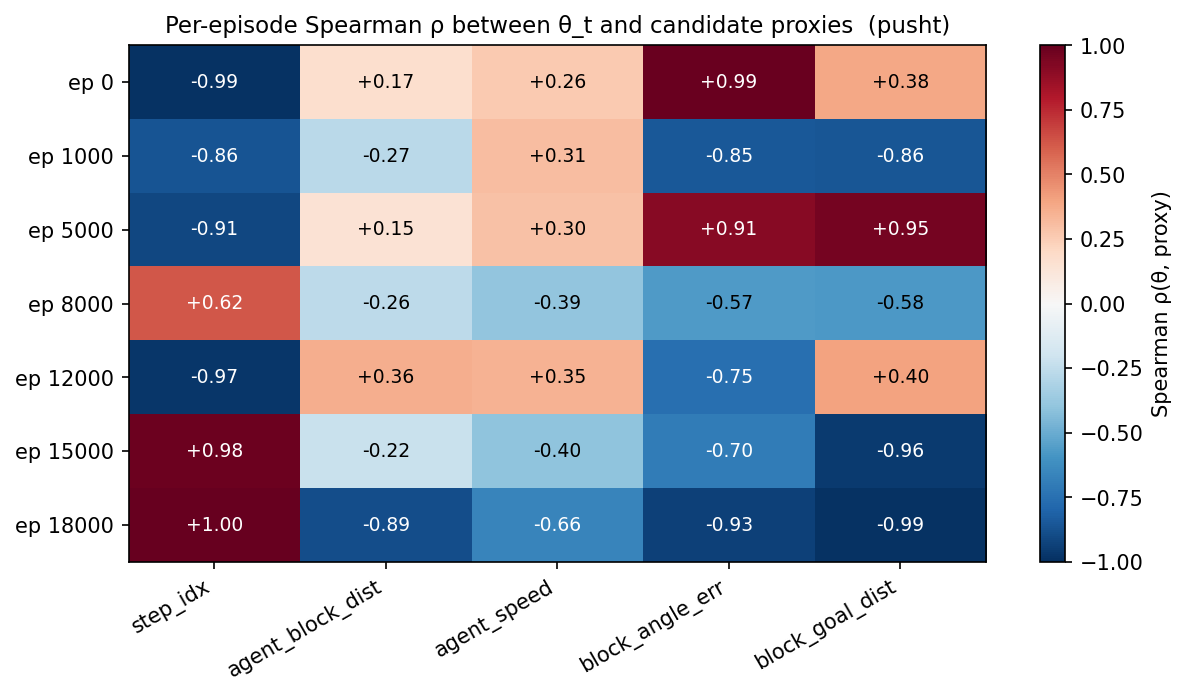}
\caption{Per-episode signed Spearman $\rho$ between $\theta_t$ and each candidate proxy on Push-T SD-JEPA ($k_{\mathrm{prog}}{=}2$). Red $=$ positive, blue $=$ negative; sign tracks the orientation of the angular trajectory in that episode. Episode 8000 (the regression case where the agent reverses) inverts the clock correlation.}
\label{fig:theta-corr-pusht}
\end{figure}

\begin{table}[h]
\centering
\small
\begin{tabular}{rccc}
\toprule
$t$ & $\overline{\cos}(z[t])$ & $\overline{\cos}(\zp[t])$ & $\overline{\cos}(\zc[t])$ \\
\midrule
5  & $0.014$  & \textbf{0.612} & $0.009$ \\
25 & $-0.016$ & \textbf{0.521} & $-0.021$ \\
50 & $-0.016$ & \textbf{0.371} & $-0.019$ \\
\bottomrule
\end{tabular}
\caption{Cross-episode mean cosine similarity at fixed $t$ on a Push-T SD-JEPA ($k{=}2$) checkpoint, seven held-out episodes spanning the dataset's episode-id range. Confirms the disjoint-supports asymmetry (Prop.~\ref{prop:disjoint}) at the second-moment level; the t-SNE in Fig.~\ref{fig:latent-geometry-body} shows the same pattern qualitatively.}
\label{tab:cosine-diagnostic}
\end{table}

\begin{figure}[h]
\centering
\includegraphics[width=0.85\linewidth]{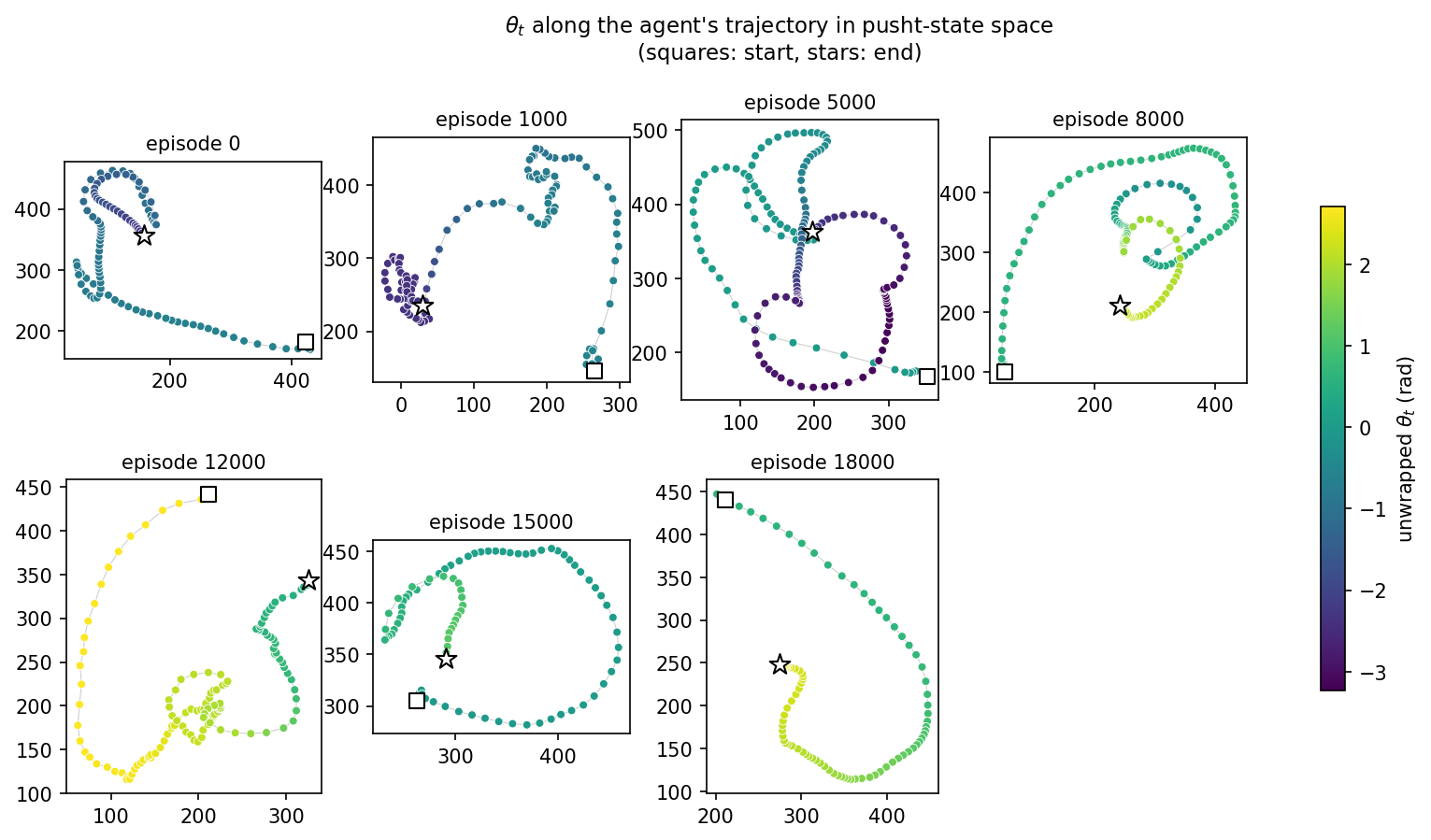}
\caption{Push-T agent trajectories in raw state space (agent $(x, y)$ in pixel-equivalent units), seven held-out episodes coloured by $\theta_t$. The angular gradient broadly tracks the agent's progression along its physical path, with notable deviations on the cyclic-motion episodes.}
\label{fig:state-traj-pusht}
\end{figure}

\subsection{Push-T at $k_{\mathrm{prog}}{=}8$}

The $k=8$ checkpoint shows a substantially cleaner $S^1$ progression manifold: $r \in [0.62, 0.77]$ with per-episode std $\in [0.04, 0.07]$, angular span $6.1$ rad (effectively full-circle coverage). The clock advantage seen at $k=2$ also narrows: mean $|\rho|$ vs.\ \texttt{step\_idx} drops from $0.91$ to $0.56$, with \texttt{block\_angle\_err} and \texttt{block\_goal\_dist} reaching $|\rho|\approx 0.56$ as well. At $k=8$ $\theta$ is less time-like and more cyclic-phase-like.

\subsection{Reacher at $k_{\mathrm{prog}}{=}8$}

\begin{figure}[h]
\centering
\includegraphics[width=0.85\linewidth]{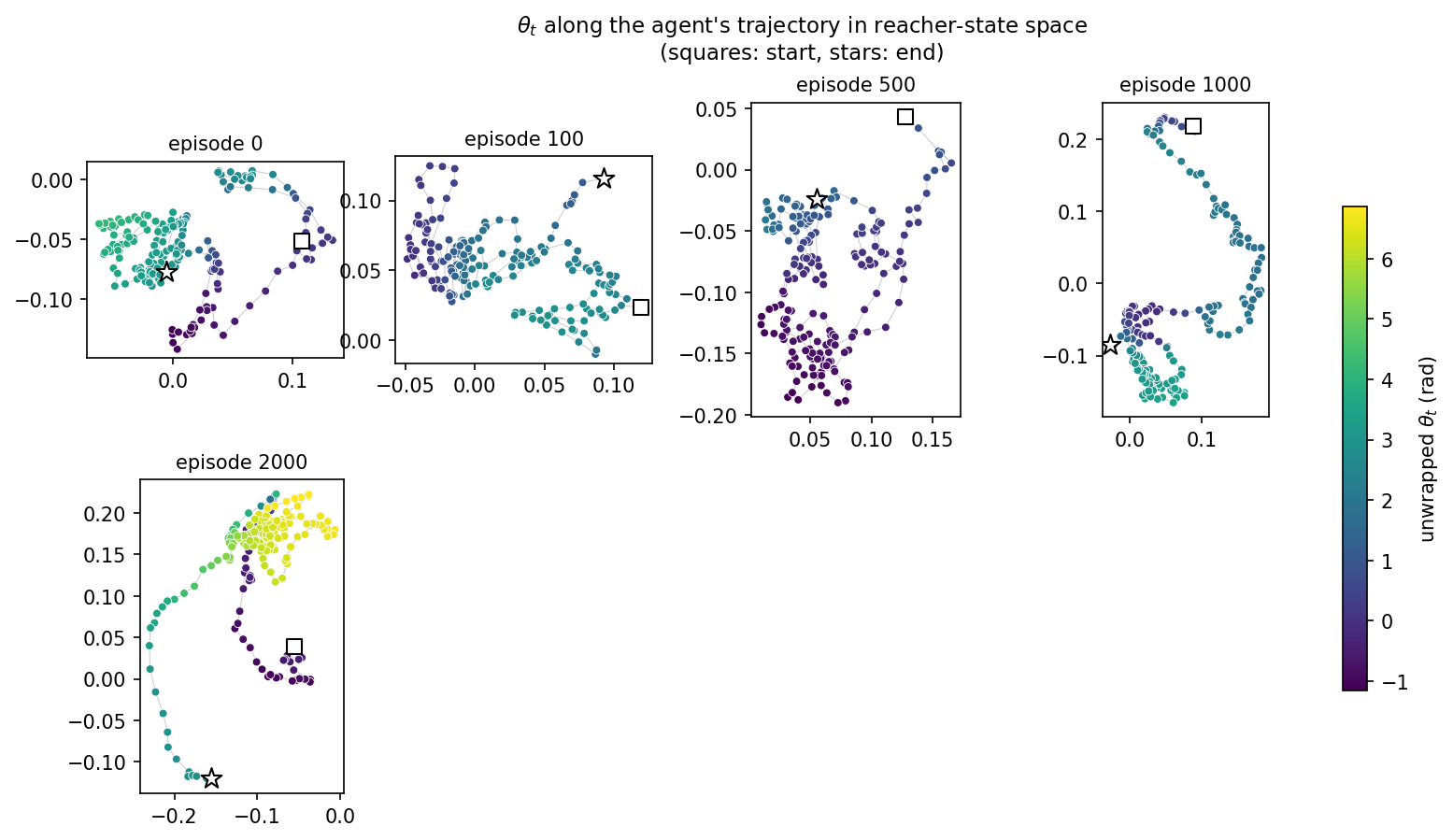}
\caption{Reacher end-effector trajectories (computed from \texttt{finger\_pos}) for five held-out episodes, coloured by $\theta_t$. Two episodes (500 and 1000) show very strong individual coupling between $\theta$ and the end-effector-to-target distance ($|\rho|=0.92$ and $0.79$), the cleanest phase-coordinate cases in the cohort.}
\label{fig:state-traj-reacher}
\end{figure}

Mean $|\rho|$ on Reacher is $0.53$ for the clock and $0.49$ for \texttt{ee\_target\_dist}; on individual episodes the latter \emph{beats} the clock. \texttt{joint\_speed} is essentially uncorrelated with $\theta$ ($|\rho|\approx 0.08$): $\theta$ is not encoding ``how fast is the arm moving.'' The $3.3$-point gain from the full-latent planning cost on Reacher (\S\ref{sec:main-results}) is consistent with $\zp$ adding usable information to the cost when $\theta$ is well coupled to the end-effector-to-target geometry.

\subsection{Two-Room at $k_{\mathrm{prog}}{=}8$, the cyclic-phase example}

\begin{figure}[h]
\centering
\includegraphics[width=0.85\linewidth]{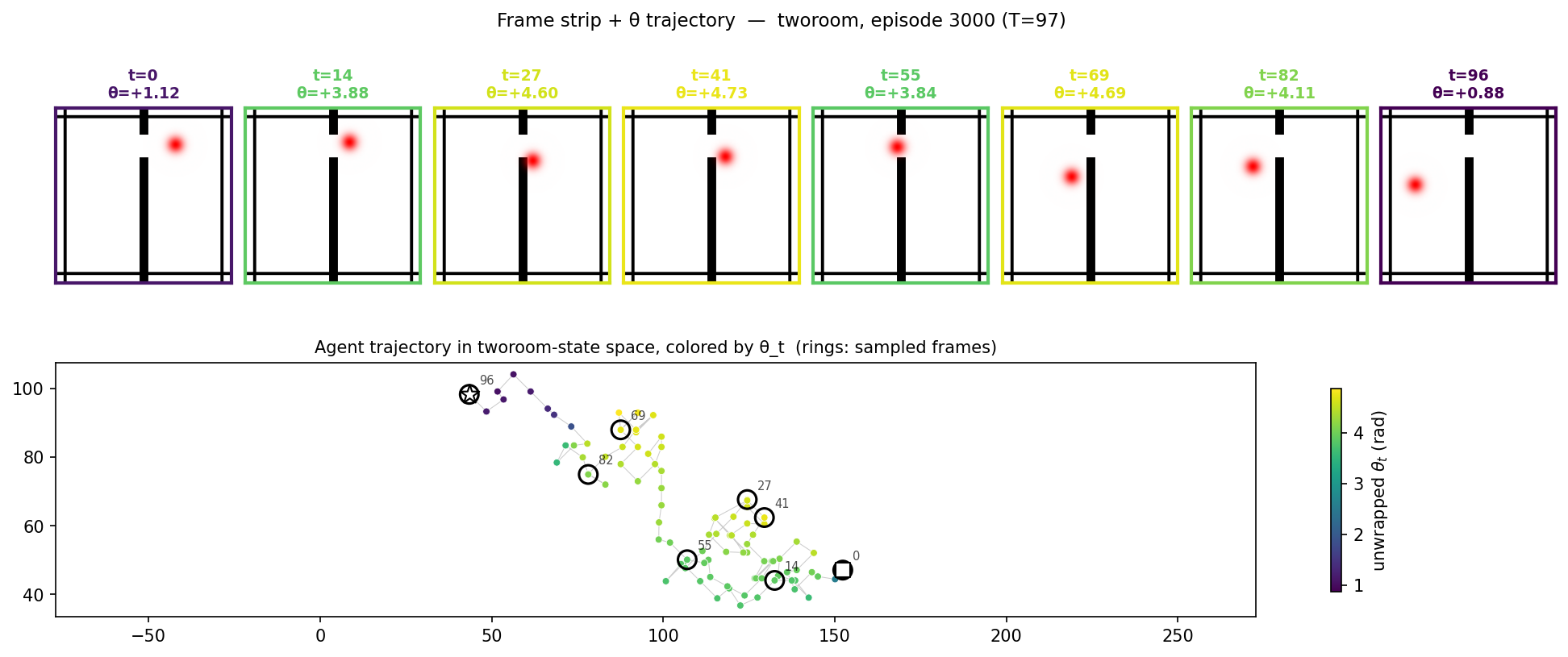}
\caption{Two-Room episode 3000: the agent crosses from the right room into the left, \emph{then comes back}. $\theta_t$ traces a U-shape ($1.12 \to 4.73 \to 0.88$); Spearman returns $\rho \approx 0$ for that episode even though $\theta$ is visibly tracking the agent's spatial phase. Spearman correlation, which assumes a monotone relation, misses this kind of cyclic phase.}
\label{fig:tworoom-frame-strip}
\end{figure}

The mean Two-Room $|\rho|$ is $0.50$ for the clock and $0.34$ for \texttt{agent\_target\_dist}, but the per-episode picture is bimodal: ep 1000 is clock-like ($\rho=+0.88$), ep 7000 is fully task-coupled ($\rho=-0.90$), ep 3000 (Fig.~\ref{fig:tworoom-frame-strip}) is the cyclic case where Spearman fails. Mutual information or circular correlation would tighten these estimates; we leave that to follow-up work.

\subsection{OGBench-Cube at $k_{\mathrm{prog}}{=}8$}

OGBench-Cube is the only environment in our suite where a task-physical signal beats the elapsed-time clock in correlation with $\theta_t$ (mean $|\rho|$ with \texttt{block\_target\_dist} is $0.59$ vs $0.47$ with \texttt{step\_idx}). It is also the environment with the most varied per-episode behaviour: spatial-distance encoding for reaches, gripper-state encoding for picks, and U-cycle encoding for U-shaped paths. The frame strip on episode 500 (Fig.~\ref{fig:cube-frame-strip}) and the 3D t-SNE projections (Fig.~\ref{fig:cube-tsne-3d}) capture this variability.

\begin{figure}[h]
\centering
\includegraphics[width=0.85\linewidth]{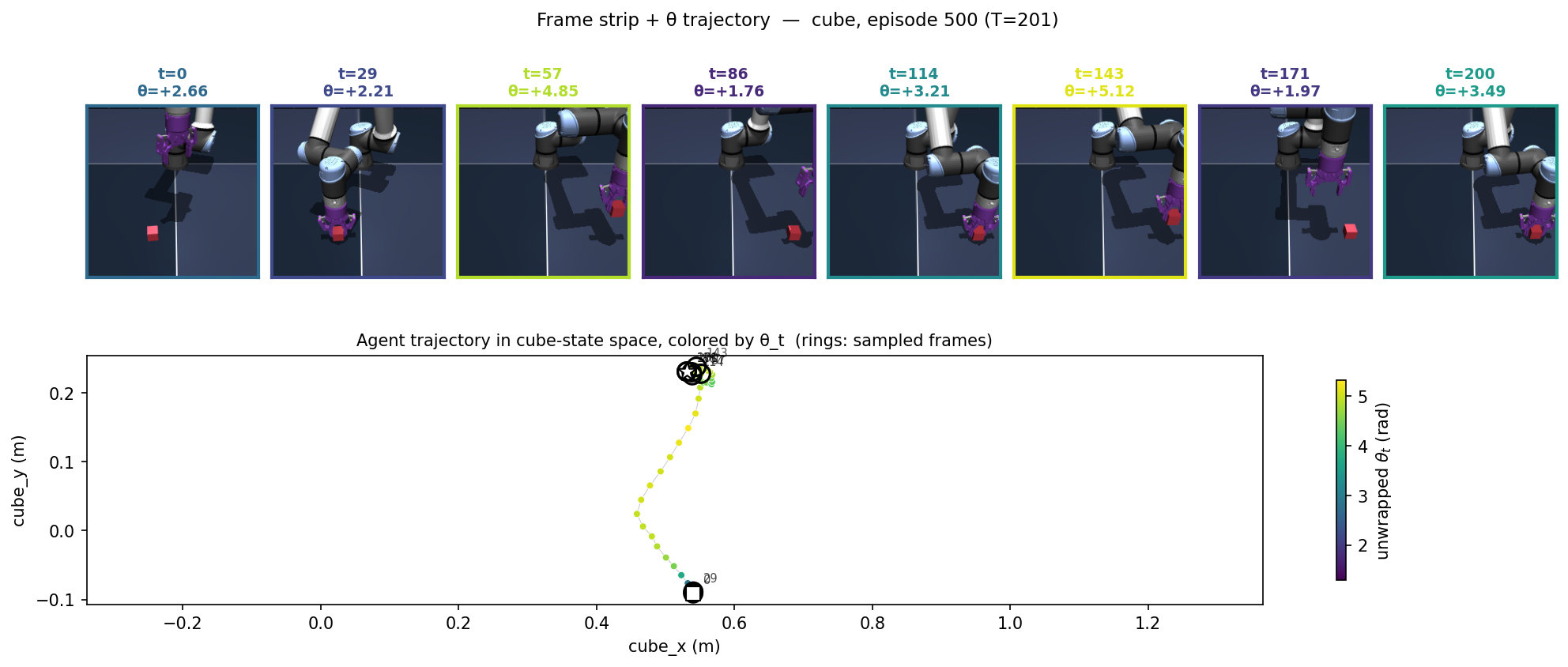}
\caption{OGBench-Cube episode 500 (a multi-phase pick): $\theta_t$ traces a non-monotone trajectory ($2.46 \to 4.65 \to 1.78 \to 3.21 \to 1.13 \to 1.97 \to 3.49$) tracking the gripper-contact / gripper-opening signals ($|\rho| = 0.76$ and $0.79$). Different cube episodes expose $\theta$ encoding \emph{different} physical phases per episode (spatial distance for reaches, gripper state for picks, U-cycle phase for tracebacks), the strongest single piece of evidence that $\theta$ is an adaptive task-phase coordinate rather than a wrapped clock.}
\label{fig:cube-frame-strip}
\end{figure}

\begin{figure}[h]
\centering
\begin{minipage}{0.49\linewidth}
\centering
\includegraphics[width=\linewidth]{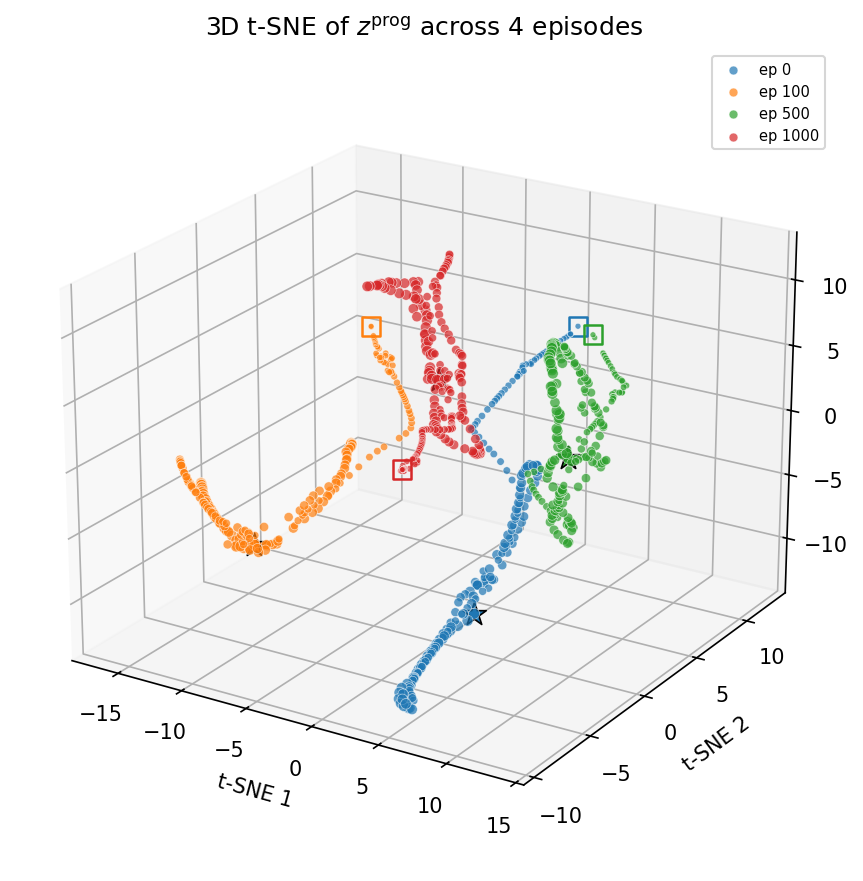}
\subcaption{$\zp$, 3D t-SNE: each episode forms a compact branch in a distinct direction.}
\end{minipage}
\hfill
\begin{minipage}{0.49\linewidth}
\centering
\includegraphics[width=\linewidth]{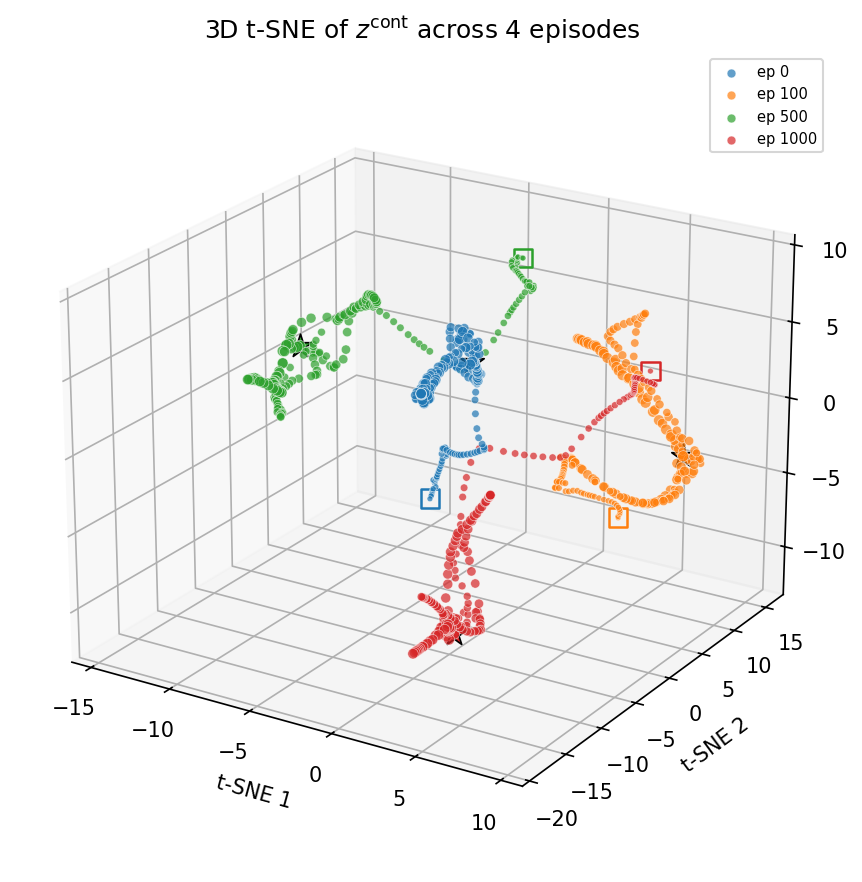}
\subcaption{$\zc$, 3D t-SNE: cleaner per-episode separation than $\zp$, even more pronounced than in 2D.}
\end{minipage}
\caption{3D t-SNE projections of the OGBench-Cube SD-JEPA latent at $k_{\mathrm{prog}}{=}8$, four held-out episodes. The third dimension makes the per-episode structure substantially more legible; the same disjoint-supports asymmetry observed on Push-T (Fig.~\ref{fig:latent-geometry-body}) holds here.}
\label{fig:cube-tsne-3d}
\end{figure}

\section{Supplementary perturbation analyses}
\label{app:perturb}

\paragraph{Brief background on head-direction cells.} For readers unfamiliar with the neuroscience analogy mentioned in \S\ref{sec:reading}: head-direction (HD) cells are neurons in the rodent hippocampal-entorhinal system that encode the animal's facing direction in the world, independent of location \citep{taube1990hd, knierim2020hippocampus}. Each cell fires preferentially when the head is oriented in a specific allocentric direction; together they form a 1-D circular code where a population peak travels around a circle as the animal turns. Two properties make the parallel with $\theta_t$ natural: (i)~a scalar circular coordinate (a 360$^\circ$ heading is collapsed to one moving peak), and (ii)~re-mapping under perturbation, when the environment is teleported, rotated, or otherwise restructured, HD-cell populations re-anchor their peak to the new visual cues and then track coherently from there, the same kind of structural re-localisation we observe in $\theta_t$ after our teleport-and-continue perturbation. We do not claim our latent rediscovered the biological circuit; the parallel is qualitative, two systems converging on a similar geometric solution (a single circular phase coordinate $+$ re-mapping under sensory restructuring) for what is broadly the same problem.

The main text shows the teleport-and-continue perturbation
(Figure~\ref{fig:teleport-continue}), in which the
``moment'' (latent surprise spike) and the ``meaning''
(post-perturbation relocalisation in $\theta$) come apart cleanly.
Three complementary perturbation modes confirm and refine that
picture. All run on Push-T SD-JEPA ($k{=}2$) ($k_{\text{prog}} = 2$,
seed $3072$, epoch 10), processed by the same encoder with no
test-time fine-tuning.

\paragraph{Splice (Figure~\ref{fig:surprise_splice}).}
We replace all frames of episode $0$ from step $50$ onward with frames
from episode $1000$. Unlike the main-text teleport-and-continue, where the planner keeps acting from the swapped frame, the splice replays the donor episode's own frames, so the model only observes and never acts on the new scene.
Result: a $\approx 0.17$ rad surprise spike at the boundary
($\sim 8\times$ baseline mean), followed by sustained elevated noise
across steps $50$--$70$ as the encoder processes the unfamiliar
trajectory. The post-perturbation $\theta_t$ diverges from the clean
baseline but does not pin to a single new sector; instead it traces a
new path consistent with the new episode's progression. This is a
weaker but still positive surprise signal.

\begin{figure}[h]
\centering
\includegraphics[width=0.85\linewidth]{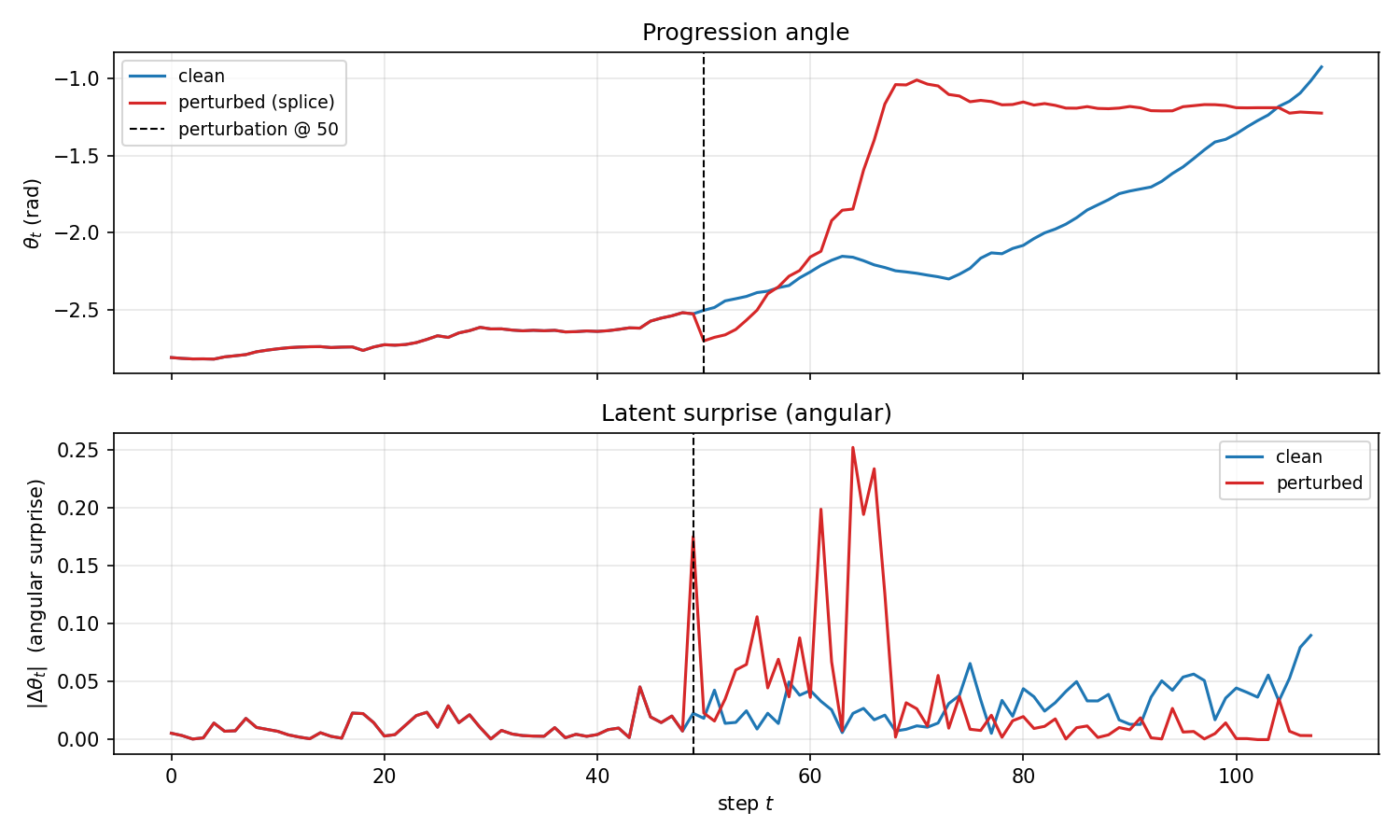}
\caption{Splice perturbation on Push-T SD-JEPA ($k{=}2$). Episode $0$'s frames are
replaced with episode $1000$'s frames from step $50$ onward. The
spike magnitude ($\approx 0.17$ rad, $\sim 8\times$ baseline) is more
modest than in the teleport-and-continue case
(Fig.~\ref{fig:teleport-continue}).}
\label{fig:surprise_splice}
\end{figure}

\paragraph{Single-frame teleport (Figure~\ref{fig:surprise_teleport}).}
We insert a single frame from a structurally different scene at step
$50$ of episode $0$, then resume episode $0$. The angular surprise
spikes to $\approx 0.20$ rad at the substituted frame and returns to
baseline within a single step, a clear example of localised
surprise with immediate recovery. The model registers the anomaly,
reports a sharp $\theta$ jump, and resumes normal localisation as soon
as the genuine observation stream returns. This mode tests
\emph{recovery dynamics}: when the perturbation is transient,
relocalisation is also transient.

\begin{figure}[h]
\centering
\includegraphics[width=0.85\linewidth]{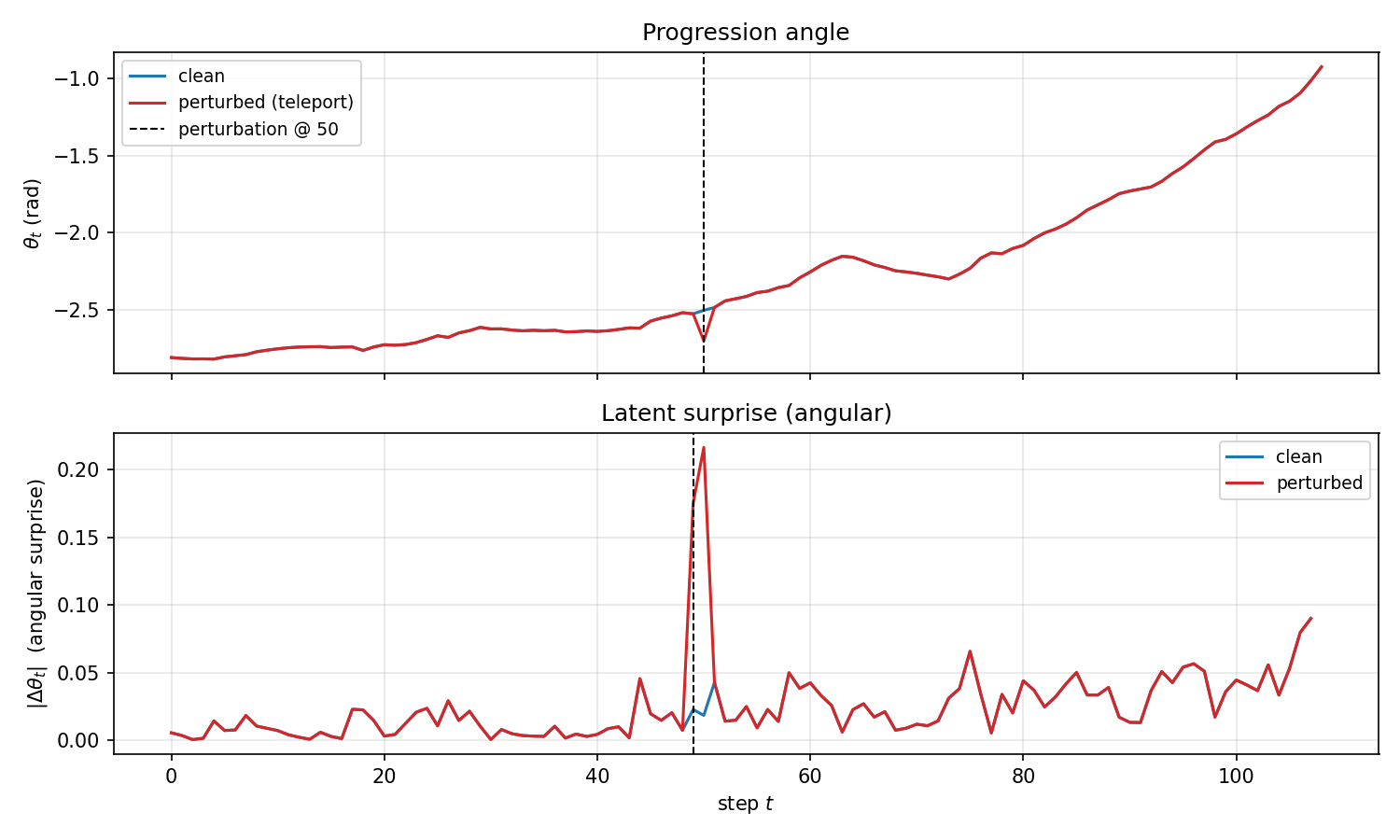}
\caption{Single-frame teleport perturbation on Push-T SD-JEPA ($k{=}2$). A single
anomalous frame is inserted at step $50$, followed by the original
episode resuming. The angular surprise spike at the perturbation
returns to baseline within one step, the cleanest illustration of
$\theta$ as a transient surprise-sensitive readout that snaps back
when the observation stream is restored.}
\label{fig:surprise_teleport}
\end{figure}

\paragraph{Reverse-window perturbation (Figure~\ref{fig:surprise_reverse}).}
We reverse a 10-frame window of episode $0$ starting at step $40$.
Critically, the set of frames presented to the encoder is unchanged ---
only the temporal order is altered. Any angular discontinuity in this
mode must therefore be a model-dynamics phenomenon, not a scene-novelty
artifact. The result: matched spikes of $\approx 0.11$ rad at the
window's leading boundary (step $40$) and $\approx 0.13$ rad at the
trailing boundary (step $50$), with elevated noise within the reversed
window. This rules out the alternative hypothesis that the surprise
spikes in the splice and teleport modes are simply registering ``a
novel scene appeared''; they are registering ``the dynamics violated
the model's local expectations'', which is the stronger and more
interesting claim.

\begin{figure}[h]
\centering
\includegraphics[width=0.85\linewidth]{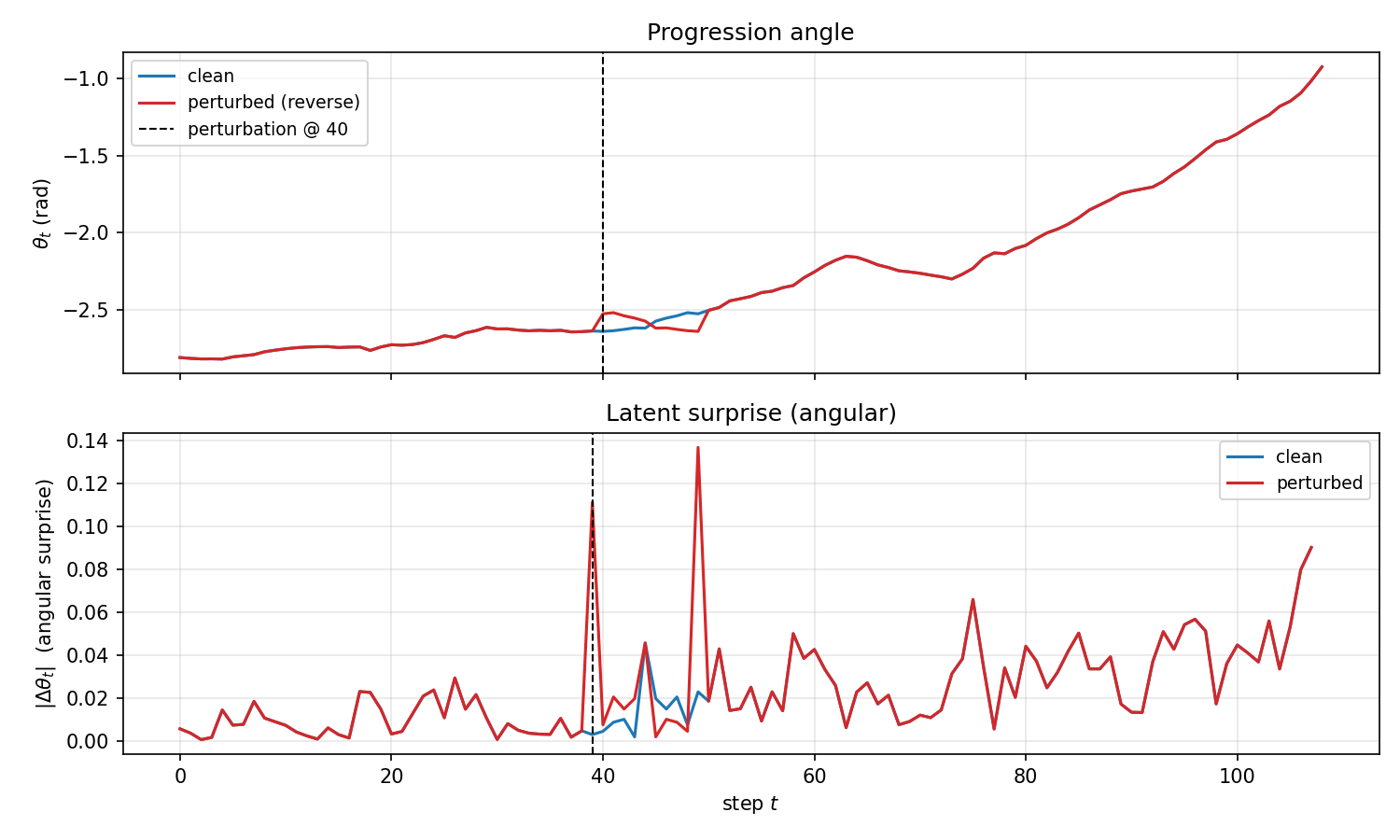}
\caption{Reverse-window perturbation on Push-T SD-JEPA ($k{=}2$). We reverse 10 frames
of episode $0$ starting at step $40$; the set of pixels presented is
unchanged, only their temporal order. Matched angular surprise spikes
appear at \emph{both} boundaries of the reversed window
(${\sim}0.11$~rad at step 40, ${\sim}0.13$~rad at step 50), confirming
that the surprise signal is a response to dynamics violation rather
than scene novelty. The clean (blue) baseline is mostly hidden behind
the red perturbed curve since the two coincide outside the reversed
window.}
\label{fig:surprise_reverse}
\end{figure}

\section{Phase-event localisation vs.\ classical surprise: extended results}
\label{app:phase-align}

This appendix collects the full empirical evidence for the comparison between $|\Delta\theta_t|$ and z-MSE as surprise signals on OGBench-Cube (\S\ref{sec:latent-diagnostics}, Tab.~\ref{tab:phase-auroc}, Fig.~\ref{fig:phase-overlay-cube}). The main finding is that the two metrics measure complementary things: $|\Delta\theta_t|$ is the right tool for semantic-event localisation, z-MSE is the right tool for raw magnitude-anomaly detection, and combining them does not improve detection on either test.

\subsection{Phase-alignment summary across all 40 episodes}

\begin{figure}[h]
\centering
\includegraphics[width=0.98\linewidth]{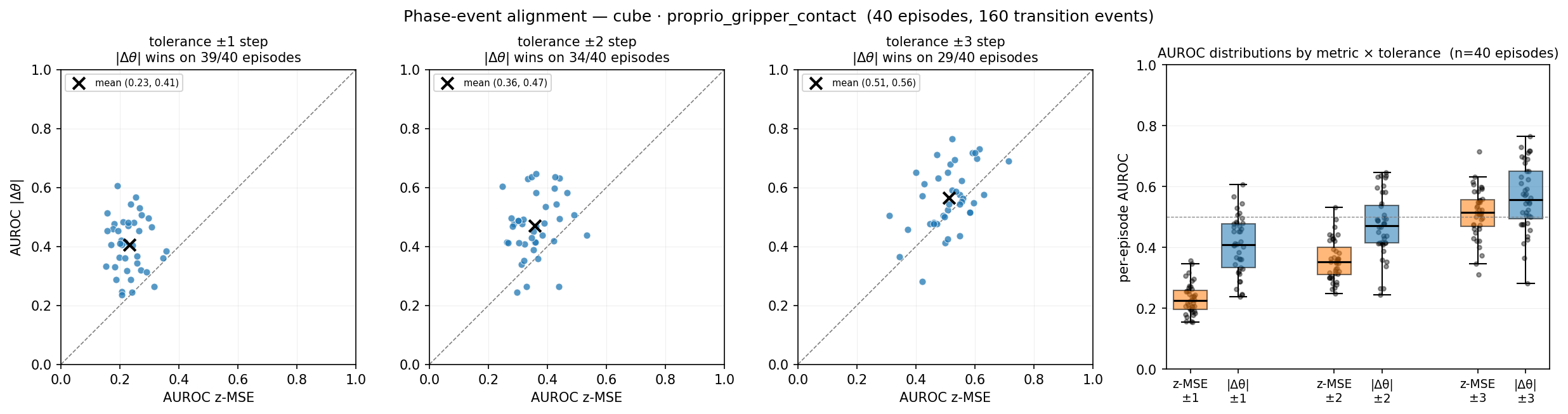}
\caption{Phase-event-alignment AUROC summary on OGBench-Cube (40 held-out episodes, 160 ground-truth gripper-contact events). The first three panels are per-episode AUROC scatter plots (one dot per episode, $x =$ AUROC of z-MSE, $y =$ AUROC of $|\Delta\theta|$) at tolerances $\pm 1, \pm 2, \pm 3$ steps; the diagonal $y = x$ marks parity. The fourth panel is a side-by-side box-and-strip plot of the per-episode AUROC distributions for each metric and tolerance. At every tolerance the $|\Delta\theta|$ cluster sits above the diagonal; the box plot shows the IQRs barely overlap.}
\label{fig:phase-align-summary}
\end{figure}

Per-episode AUROC distribution at $\pm 2$-step tolerance:

\begin{center}
\small
\begin{tabular}{lccc}
\toprule
metric             & median   & mean     & IQR              \\
\midrule
$|\Delta\theta|$   & $0.472$  & $0.470$  & $[0.415, 0.544]$ \\
z-MSE              & $0.353$  & $0.358$  & $[0.312, 0.421]$ \\
\bottomrule
\end{tabular}
\end{center}

Both AUROC distributions sit below $0.5$ at the tight $\pm 1$-step tolerance: z-MSE peaks during the stable-contact phase (when the cube is being translated while gripped) rather than at the transitions themselves, and only ${\sim}11\%$ of all steps are positive labels. $|\Delta\theta|$'s peaks are consistently \emph{closer} to transitions even when not exactly aligned, so widening the tolerance window catches more of those near-aligned peaks and pushes $|\Delta\theta|$ above $0.5$ before z-MSE follows. The right reading is the gap, not the absolute level.

\subsection{Action-corruption AUROC: the magnitude-anomaly negative control}

To verify that the two metrics genuinely measure different things, we run a complementary test where the corruption is a magnitude anomaly rather than a semantic event. We replace a contiguous $25\%$ band of each episode's actions with a random shuffle and ask each metric to localise the corrupted band as a per-step random-segment AUROC.

\begin{center}
\small
\begin{tabular}{llccc}
\toprule
env       & corruption     & z-MSE              & $|\Delta\theta|$ & combined (rank-sum) \\
\midrule
Reacher   & action shuffle & $\mathbf{0.771}$   & $0.657$          & $0.736$             \\
Reacher   & other-episode  & $\mathbf{0.808}$   & $0.663$          & $0.763$             \\
Push-T    & action shuffle & $\mathbf{0.598}$   & $0.496$          & $0.553$             \\
Cube      & action shuffle & $\mathbf{0.696}$   & $0.659$          & $0.694$             \\
\bottomrule
\end{tabular}
\end{center}

z-MSE wins this test by construction: the action-shuffle injects a magnitude anomaly that a $192$-dimensional MSE detects more sensitively than a $1$-dimensional phase scalar; the per-step random-segment label rewards information bandwidth, not semantic alignment. The combined rank-sum score does not beat z-MSE on any cell; the two metrics are not strictly orthogonal but are clearly answering different questions.

\subsection{Regime-change segmentation (change-point detection)}

A second formulation of the same hypothesis: at semantic phase events, $\theta_t$ exhibits a regime shift (different mode of evolution before vs.\ after) while z-MSE just spikes and returns to baseline. We run BinSeg with $K = 4$ requested change points on the $40$ cube episodes and compute precision/recall/F1 against the ground-truth contact transitions with greedy 1-to-1 matching at tolerance $\pm$tol.

\begin{center}
\small
\begin{tabular}{lccc}
\toprule
metric                                    & F1 $\pm 1$         & F1 $\pm 2$         & F1 $\pm 3$         \\
\midrule
$|d\theta/dt|$ (observed angular vel.)    & $0.394$            & $0.613$            & $\mathbf{0.812}$   \\
z-MSE                                     & $\mathbf{0.481}$   & $\mathbf{0.769}$   & $0.775$            \\
$|\Delta\theta|$ (prediction error)       & $0.425$            & $0.725$            & $0.744$            \\
\bottomrule
\end{tabular}
\end{center}

\begin{figure}[h]
\centering
\includegraphics[width=0.85\linewidth]{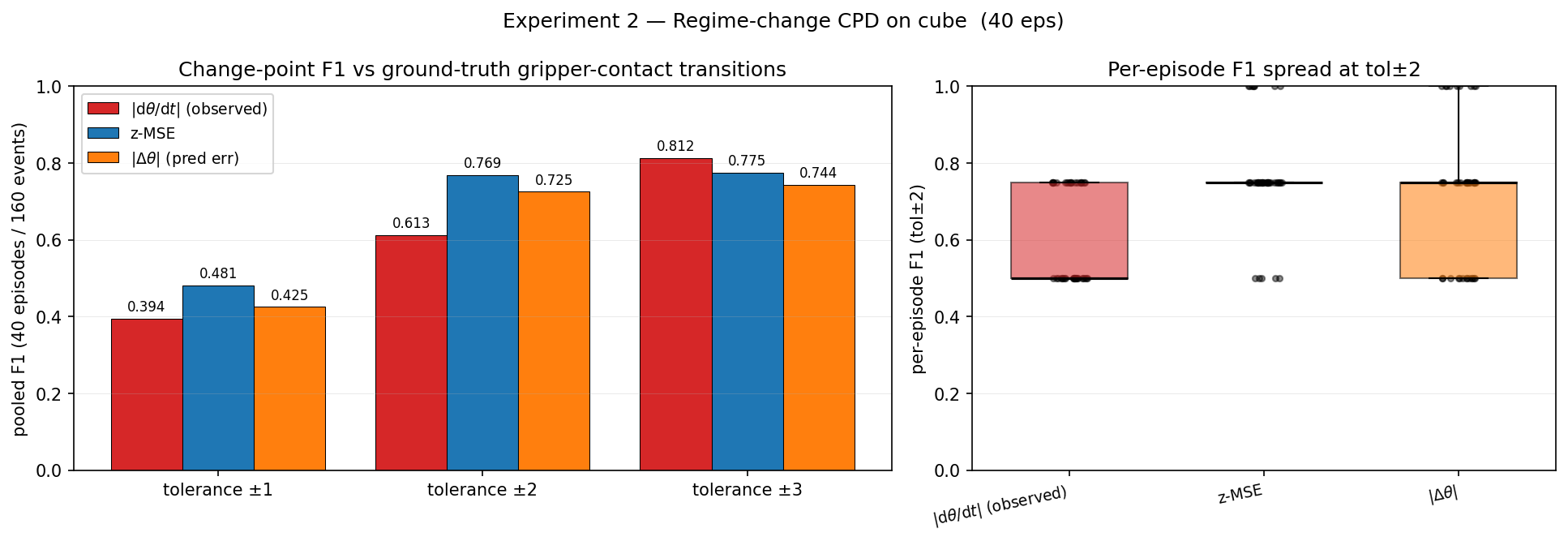}
\caption{Change-point F1 summary across 40 cube episodes for the three signals at three tolerances. At loose tolerance ($\pm 3$), the \emph{observed} angular-velocity trace $|d\theta/dt|$ beats both prediction-error metrics, supporting the reading that the $\theta$ trajectory itself is naturally segmented by the task. At tighter tolerance, BinSeg locates spike-and-return boundaries (z-MSE) more precisely than smoother regime boundaries.}
\label{fig:regime-change-summary}
\end{figure}

The change-point result is mixed: at loose tolerance the observed angular-velocity trace beats both prediction-error metrics, but at tight tolerance z-MSE's spike-and-return pattern produces sharper change-point boundaries that BinSeg locates more accurately. This is consistent with the AUROC reading: $|\Delta\theta|$ excels at peak alignment with semantic events, z-MSE excels at sharp magnitude-spike segmentation; both are useful and the choice depends on what the practitioner wants to detect.

\subsection{Linear-probe information density: full per-feature table}

\begin{figure}[h]
\centering
\includegraphics[width=0.98\linewidth]{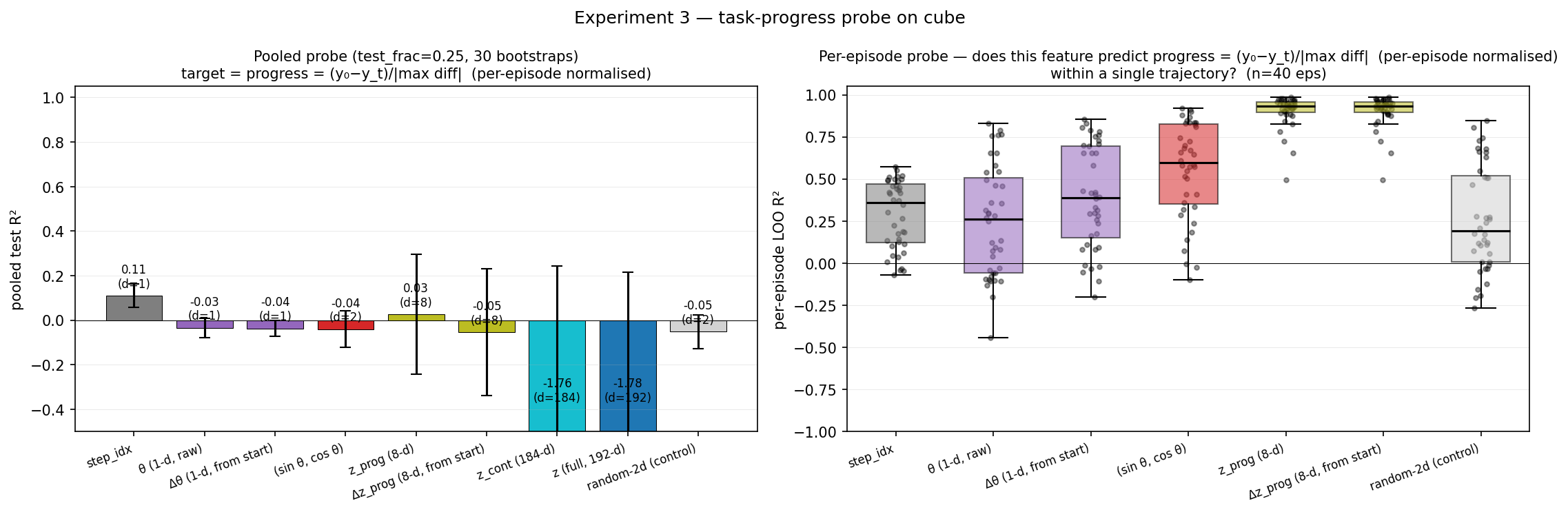}
\caption{Linear-probe R$^2$ on cube task progress (cube-to-target distance, per-episode normalised). Left: pooled probe across episodes; only \texttt{step\_idx} is positive, the high-dim features overfit catastrophically. Right: within-episode probe (leave-one-step-out CV per episode, mean over $40$ episodes); $\zp$ ($8$-d) reaches mean R$^2 = 0.91$ and is positive on $100\%$ of episodes; $(\sin\theta, \cos\theta)$ packs $55.5\%$ of the variance into 2 dimensions, well above the random-2-d control.}
\label{fig:probe}
\end{figure}

We fit Ridge probes from each latent feature to a per-episode-normalised cube-to-target progress target. The pooled regime leaves $25\%$ of episodes out and fits on the rest ($30$ bootstraps); the per-episode regime fits within each of $40$ held-out episodes with leave-one-step-out CV. Figure~\ref{fig:probe} contrasts the pooled and per-episode regimes; the full per-feature table is below.

\begin{center}
\small
\begin{tabular}{lcccc}
\toprule
feature                              & dim & pooled R$^2$         & per-episode mean R$^2$  & frac per-ep R$^2 > 0$ \\
\midrule
\texttt{step\_idx} (clock baseline)  & $1$   & $\mathbf{+0.110}$    & $+0.291$                 & $90\%$               \\
$\theta$ (raw, $1$-d)                & $1$   & $-0.034$             & $+0.243$                 & $67\%$               \\
$\Delta\theta$ (from start, $1$-d)   & $1$   & $-0.037$             & $+0.392$                 & $85\%$               \\
$(\sin\theta, \cos\theta)$           & $2$   & $-0.040$             & $\mathbf{+0.555}$        & $92.5\%$             \\
$\zp$                                & $8$   & $+0.026$             & $\mathbf{+0.905}$        & $\mathbf{100\%}$     \\
$\Delta\zp$ (from start)             & $8$   & $-0.054$             & $+0.905$                 & $100\%$              \\
$\zc$                                & $184$ & $-1.762$ (overfit)   & ---                      & ---                  \\
$\z$ (full)                          & $192$ & $-1.777$ (overfit)   & ---                      & ---                  \\
random $2$-d projection (control)    & $2$   & $-0.052$             & $+0.263$                 & $77.5\%$             \\
\bottomrule
\end{tabular}
\end{center}

The pooled probe makes one further point clearly: the compass is a \emph{per-trajectory} phase coordinate, not a globally calibrated absolute-distance estimate. The two regimes ask different questions, and the answers are consistent with the design intent ($\zp$ encodes within-episode progression, while episode-specific scene content lives in $\zc$).

\begin{figure}[h]
\centering
\includegraphics[width=0.98\linewidth]{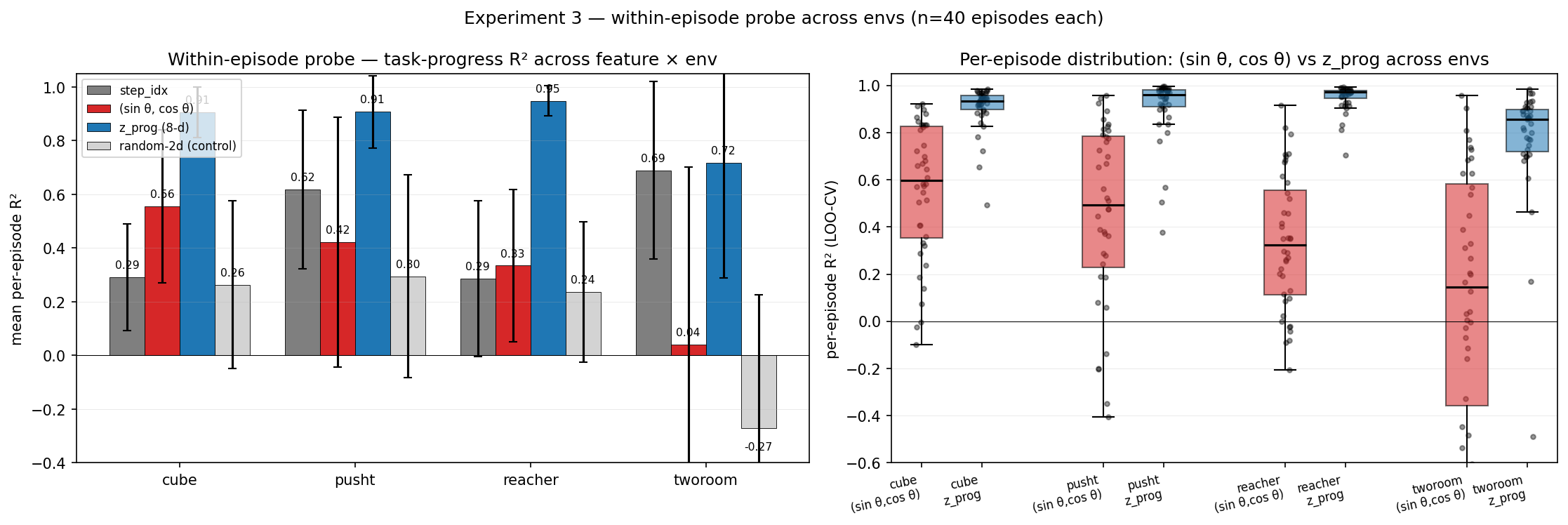}
\caption{Cross-environment per-episode probe ($40$ episodes per env, LOO-CV) against the env-specific natural target-distance signal. Left: mean R$^2$ across feature $\times$ env; $\zp$ ($8$-d, blue) wins everywhere using only $4.2\%$ of the latent dim, with the largest probe-vs-clock gap on Reacher. Right: per-episode R$^2$ distribution, $(\sin\theta, \cos\theta)$ vs $\zp$ across the four envs. Mean values are in Table~\ref{tab:probe-crossenv}.}
\label{fig:probe-crossenv}
\end{figure}

\begin{table}[h]
\centering
\small
\begin{tabular}{lccc}
\toprule
                       & median   & mean     & IQR              \\
\midrule
$\zp$                  & $0.932$  & $0.905$  & $[0.85, 0.98]$   \\
$(\sin\theta,\cos\theta)$ & $0.596$  & $0.555$  & $[0.46, 0.69]$   \\
\texttt{step\_idx}     & $0.360$  & $0.291$  & $[0.18, 0.45]$   \\
\bottomrule
\end{tabular}
\caption{Per-episode R$^2$ distribution for the main features (40 cube episodes, leave-one-step-out CV). $\zp$'s IQR sits well above the clock's, with no overlap.}
\label{tab:probe}
\end{table}

\end{document}